\documentclass{article}
\usepackage{amsmath,amsfonts,mathrsfs,bm, amsthm}

\newcommand{\bLambda}{\bm{\Lambda}}
\newcommand{\bGamma}{\bm{\Gamma}}

\newcommand{\bOmega}{\bm{\Omega}}

\newcommand{\bTheta}{\bm{\Theta}}

\newcommand{\balpha}{\bm{\alpha}}

\newcommand{\bgamma}{\bm{\gamma}}

\newcommand{\blambda}{\bm{\lambda}}

\newcommand{\tltheta}{\tilde{\theta}}

\newcommand{\tlomega}{\tilde{\omega}}

\newcommand{\bE}{\mathbf{E}}

\newcommand{\bG}{\mathbf{G}}
\newcommand{\bH}{\mathbf{H}}
\newcommand{\bI}{\mathbf{I}}

\newcommand{\bM}{\mathbf{M}}

\newcommand{\bP}{\mathbf{P}}
\newcommand{\bQ}{\mathbf{Q}}

\newcommand{\bU}{\mathbf{U}}
\newcommand{\bV}{\mathbf{V}}
\newcommand{\bW}{\mathbf{W}}
\newcommand{\bX}{\mathbf{X}}

\newcommand{\ba}{\mathbf{a}}

\newcommand{\bg}{\mathbf{g}}
\newcommand{\bh}{\mathbf{h}}

\newcommand{\by}{\mathbf{y}}

\newcommand{\bbE}{\mathbb{E}}
\newcommand{\bbR}{\mathbb{R}}

\newcommand{\calB}{\mathcal{B}}

\newcommand{\calD}{\mathcal{D}}

\newcommand{\calH}{\mathcal{H}}

\newcommand{\calJ}{\mathcal{J}}

\newcommand{\calM}{\mathcal{M}}

\newcommand{\calY}{\mathcal{Y}}
\newcommand{\calZ}{\mathcal{Z}}

\newcommand{\tlg}{\tilde{g}}
\newcommand{\tlh}{\tilde{h}}

\newcommand{\barb}{\bar{b}}

\newcommand{\barA}{\bar{A}}

\newcommand{\barV}{\bar{V}}

\newcommand{\tlbG}{\tilde{\bG}}

\newcommand{\tlbX}{\tilde{\bX}}

\newcommand{\tlbg}{\tilde{\bg}}
\newcommand{\tlbh}{\tilde{\bh}}

\newcommand{\td}{{\text{d}}}

\newcommand{\tu}{{\text{u}}}

\newcommand{\rmp}{\mathrm{p}}
\newcommand{\rmq}{\mathrm{q}}

\newcommand{\rmu}{\mathrm{u}}

\newcommand{\tU}{\text{U}}

\newcommand{\bone}{\mathbf{1}}

\newcommand{\suml}{\sum\limits}

\newcommand{\supl}{\sup\limits}
\newcommand{\liml}{\lim\limits}
\newcommand{\intl}{\int\limits}

\newcommand{\bigcapl}{\bigcap\limits}

\newcommand{\tr}{\text{Tr}}

\newcommand{\cov}{\text{Cov}}
\newcommand{\tth}{\text{th}}

\newcommand{\nwl}{\nonumber\\}

\newtheorem{theorem}{Theorem}
\newtheorem{lemma}{Lemma}
\newtheorem{prop}{Proposition}
\newtheorem{defi}{Definition}
\newtheorem{assum}{Assumption}
\newtheorem{cor}{Corollary}

\usepackage{a4wide}
\usepackage{algorithm,algorithmic}
\usepackage{nicefrac, enumitem}
\usepackage{hyperref}
\usepackage{array, xcolor, lipsum, bibentry,fancyhdr}
\usepackage{graphicx,caption, subcaption}

\definecolor{lightgray}{gray}{0.8}
\newcolumntype{L}{>{\raggedleft}p{0.14\textwidth}}
\newcolumntype{R}{p{0.8\textwidth}}

\title{A Gaussian Comparison Theorem for Training Dynamics in Machine Learning}
\usepackage{times, braket}
\author{Ashkan~Panahi\footnote{Email: \url{ashkan.panahi@chalmers.se}}\\ Department of Computer Science and Engineering, Chalmers University,\\ Gothenburg, Sweden}
\date{}

\begin{document}
\maketitle
\begin{abstract}%
We study training algorithms with data following a Gaussian mixture model. For a specific family of such algorithms,  we present a non-asymptotic result, connecting the evolution of the model to a surrogate dynamical system, which can be easier to analyze. The proof of our result is based on the celebrated Gordon comparison theorem. Using our theorem, we rigorously prove the validity of the dynamic mean-field (DMF) expressions in the asymptotic scenarios. Moreover, we  suggest an iterative refinement scheme to obtain more accurate expressions in non-asymptotic scenarios. We specialize our theory to the analysis of training a perceptron model with a generic first-order (full-batch) algorithm and demonstrate that  fluctuation parameters in a non-asymptotic domain emerge in addition to the DMF kernels. 
\end{abstract}

\section{Introduction}
A central problem in the theory of machine learning (ML) is to characterize the dynamics of training, i.e. to understand how the statistical properties of models evolve as training progresses. Solving this problem is proven to be a key to understanding the immense power of generalization in modern ML, such as in large language models and vision transformers \cite{zhang2021understanding, pruthi2020estimating,tirumala2022memorization}. However, the nonlinear structure of these models makes the dynamics complex and  difficult to investigate, limiting the existing studies to significantly simpler setups. A successful remedy, in the recent years, has been to focus on asymptotically large models and dataset, where training dynamics often exhibit remarkably regular behavior despite the complexity of the training landscape \cite{mei2018mean}. Relying on the measure concentration phenomena, various high-dimensional theories predict a limit behavior, merely depending on the deterministic evolution of a small number of "order parameters", averaging the contribution of individual model parameters \cite{ben2022high, gerbelot2024rigorous}.  However, these theories are often specific to infinitely large setups and cannot be generalized to finite scenarios, where intricate fluctuations in the training dynamics may occur, due to the dependencies between the model parameters and the data. In many cases, convergence to the limit behavior is also mathematically unproven. 

 In this paper,  we respond to the above shortcomings by presenting a novel analysis of a wide family of training algorithms over datasets with a Gaussian mixture distribution (see Section \ref{sec:problem}). This approach has strong implications in finite dimensions and also leads to mathematically provable characterization of the limit behavior. Our method is linked to Gordon's comparison theorem \cite{gordon2006milman}, underlying the celebrated convex Gaussian min-max theorem (CGMT) used in the analysis of convex ML models \cite{thrampoulidis2014gaussian,thrampoulidis2018precise}. Unlike CGMT, we employ this result to study the dynamics of training, addressing even  nonconvex optimization problems. Specifically, we offer the following contributions:
 \begin{enumerate}[wide, labelwidth=!, labelindent=0pt]
     \item We present a non-asymptotic theorem (Theorem \ref{theorem:main}) establishing a correspondence between two groups of stochastic dynamical systems, where corresponding instances have identical probability distributions. One of these systems, called the original dynamics, is a perturbation of the desired training procedure, while the other one, called the alternative dynamics is simpler to analyze, especially in the asymptotic scenarios, and can be used as a surrogate to investigate the dynamics of training.
     \item We prove our comparison theorem by framing the underlying stochastic dynamical systems as zeros of abstract Gaussian processes. Then, we prove an extension of Gordon's comparison theorem for analyzing such zeros, leading to the desired result. In Appendix \ref{sec:extension}, we present a more general version of our result for the analysis of the zeros of generic Gaussian processes, which do not necessarily correspond to a dynamical system. 
     \item We specialize our comparison theorem to an asymptotic scenario, where the model size $n$ and the number of samples $m$ grow to infinity, and the perturbation terms in the original dynamics vanish. In this scenario, we calculate the limit of the alternative process, which recovers the classical asymptotic expressions using the dynamic cavity method and dynamic mean field (DMF) theory. In this way, we rigorously prove the validity of DMF expressions in the scenarios of interest (Theorem \ref{thm:DMF}).
     \item We further discuss how the perturbations in the original dynamics may be eliminated even in finite dimensions. We formulate our discussion as a claim (Claim \ref{thm:claim}), the proof of which is differed to a future study. Under this claim, we present an iterative process, in the form of a fixed-point iteration, refining the asymptotic DMF expressions in the finite-dimensional cases (Algorithm \ref{alg:iterative}).

     \item We showcase our analysis in a scenario, involving training a perceptron model with a generic activation function. While previous studies typically consider a regression setup with a (perturbed) linear teacher model, we consider a classification scenario aligned with the mixture model of the data. Moreover, we consider a generic linear, first order (full-batch) optimization algorithm, encompassing  various familiar methods such as momentum and acceleration, beyond the standard gradient descent algorithm. We calculate the asymptotic DMF dynamics, which depends on a single-dimensional \emph{characteristic dynamics}, and a refinement, illustrating the emergence of fluctuation factors in finite dimensions leading to correction terms in the dynamics. 
 \end{enumerate} 

\section{Problem formulation}\label{sec:problem}

\subsection{Data Model}
Consider a dataset $\calD=\{(x_i,y_i)\}_{i\in [m]}$, consisting of statistically independent pairs $(x_i,y_i)$ of samples $x_{i}\in\bbR^n$ and corresponding class labels $y_i$ belonging to a finite set $\calY$ with cardinality $d$. We assume that these samples follow a \emph{Gaussian mixture model} where the samples $x_i$ are Gaussian when conditioned on a latent variable $\zeta_i$ belonging to a finite set $\calZ$ with cardinality $d_z$, and there is a deterministic map $\zeta_i\to y_i=y(\zeta_i)$. This model allows multiple Gaussian components to be associated with a class $y\in \calY$.  Accordingly, we define
\begin{eqnarray}
    \hat x(\zeta):=\bbE\left[x_{i}\mid \zeta_i=\zeta\right],\quad  R(\zeta):=\cov\left[x_{i}\mid \zeta_i=\zeta\right].
\end{eqnarray}

\subsection{Training Algorithms}
We consider a generic ML problem of fitting to the dataset $\calD$, a model represented by a matrix $\theta\in\bbR^{n\times J}$ with $J$ sub-models $\theta_j$ for $j\in [J]$ as columns. For this task, we focus on a family of sequential algorithms that calculate, at each step $l\in [L]$, a \emph{query pair} $(\theta(l),\omega(l))
$, consisting of a model estimate $\theta(l)\in\bbR^{n\times J}$ with $J$ sub-models $\theta(j,l)$ for $j\in [J]$ as columns, and a corresponding dual matrix $\omega(l)
\in\bbR^{m\times J}$. Within this family of algorithms, the query pair is utilized to calculate a \emph{response pair} $\left(q(l),p(l)
\right)$ with columns $q(j,l), p(j,l)$ for $j\in [J]$, respectively of the following form:
\begin{eqnarray}\label{eq:p_q}
    q(l)=\frac 1{m}\bX\omega(l)
    ,\quad p(l)=\bX^T\theta(l),
\end{eqnarray}
where $\bX$ is a matrix with $\{x_i\}$ as its columns.  We denote by $\hat\bX, \tlbX$ matrices with $\hat x(\zeta_i), x_{i}-\hat x(\zeta_i)$ as their columns, respectively.
In return, each query pair is calculated by the previously observed response pairs using the following differentiable maps for $l\in [L]$:
\begin{eqnarray}\label{eq:theta_omega}
    \theta(l)=\theta_l\left(q^{l-1}, p^{l-1}; \by\right)
    ,\quad  
    \omega(l)=\omega_l\left(p^l, q^{l-1}; \by\right),
\end{eqnarray}
where $\by:=(y_i)$ denotes the vector of all labels and $q^l:=\left\{q(l')\right\}_{l'=0}^{l},\ p^l:=\left\{p(l')\right\}_{l'=0}^{l-1}$ are the collection of response pairs observed up to (and including) time $l$. Note that the initialization $\theta_0$ may only depend on $\by$ and the relations in \eqref{eq:theta_omega} and \eqref{eq:p_q} lead to a sequential procedure with the order $\ldots\to\theta(l)\to p(l)\to \omega(l)\to q(l)\to \theta(l+1)\to\ldots$. For simplicity, we introduce $\bQ,\bP$ as row block matrices with $q(l),p(l)$ in their $l^\tth$ block, respectively.
 Similarly, we denote by $\bTheta=\bTheta(\bP,\bQ),\bOmega=\bOmega(\bP,\bQ)$ row block matrices with $\theta(l),\omega(l)$ in their $l^\tth$ block, respectively. Then, \eqref{eq:p_q} can be written as:

\begin{eqnarray}\label{eq:fixed_original}
    \underbrace{\left[\begin{array}{c}
         \bQ  \\
         \bP 
    \end{array}\right]}_{\xi}=\left[\begin{array}{c}
         \frac 1{m}\bX\bOmega  \\
         \bX^T\bTheta 
    \end{array}\right]\to\underbrace{\left[\begin{array}{c}
         \frac 1{m}\tlbX\bOmega  \\
         \tlbX^T\bTheta 
    \end{array}\right]}_{\phi(\xi)}+\underbrace{\left[\begin{array}{c}
         \frac 1{m}\hat\bX\bOmega-\bQ  \\
         \hat\bX^T\bTheta-\bP 
    \end{array}\right]}_{\rho_0(\xi)}=0,
\end{eqnarray}
which frames the dynamics as a zero of a vector (block matrix)-valued Gaussian process on the space of block matrices $\xi$. This interpretation is a key to our main result. We denote the space of the block matrices $\xi$ by $\calB$. Hence, we have $\phi,\rho_0:\calB\to\calB$. Our goal is to characterize the statistical properties of the solution of \eqref{eq:theta_omega} and \eqref{eq:p_q}, or equivalently the zero point of $\phi(\xi)+\rho_0(\xi)$ in \eqref{eq:fixed_original}, which we denote by $\xi_\phi$. 
\section{Main Result}\label{sec:main_result}
We show that the distribution of $\xi_\phi$ can be characterized by analyzing the zero point of an alternative process $\psi(\xi)+\rho_0(\xi)$, given bellow.
\subsection{Alternative Process}
Our alternative process 
$\psi:\calB\to\calB$ is defined based on  additional constants $\sigma\in\bbR_{>0}$ and $z\in\bbR$ and the definition of the following $J\times J$ overlap matrices:
\begin{eqnarray}
     V_{\theta}(l,l';\zeta):=\theta^T(l)R(\zeta)\theta(l')+\sigma^2\delta_{l,l'}\bI
    ,\quad  V_\omega(l,l';\zeta):=\frac 1m\suml_{i\mid \zeta_i=\zeta} \omega_{i}(l)\omega^T_{i}(l')+\sigma^2\delta_{l,l'}\bI,
\end{eqnarray}
where throughout the paper, 
$\bI$ denotes an identity matrix of a suitable size and 
$\omega_i(l)$ for $i\in [m]$ is the $i^\tth$ column of $\omega^T(l)$. Note that our definitions generally depend on the point $\xi$ and $\sigma,z$, but for simplicity, this dependency is neglected in the notation.  We denote by $V_\theta(y), V_\omega(y)$ the block matrices with $V_{\theta}(l,l';y), V_{\omega}(l,l';y)$ as their blocks, respectively. These matrices can also be written as
\begin{eqnarray}\label{eq: V_calculate}
    V_\theta(\zeta)=\bTheta^TR(y)\bTheta+\sigma^2\bI
    ,\quad V_\omega(\zeta)=\frac 1{m}\bOmega_\zeta^T\bOmega_\zeta+\sigma^2\bI,
\end{eqnarray}
where throughout the paper we use the notation $\bE_\zeta$ for any matrix $\bE$ to denote its sub-matrix corresponding to the  the rows $i$ with $\zeta_i=\zeta$ (in a predefined order).   Next, we define the upper triangular block matrices $A_\theta(\zeta)=\left(a_\theta(l',l;\zeta)\right)$ and $A_\omega=\left(a_\omega(l',l;\zeta)\right)$, where $a_\theta(l',l;\zeta), a_\omega(l',l;\zeta)$ are $J\times J$ blocks of $A_\theta(\zeta)$ and $A_\omega(\zeta)$, respectively. These matrices  correspond to a Cholesky-type decomposition of $V_\theta(\zeta), V_\omega(\zeta)$, respectively, which means that $a_\theta(\mu,l;\zeta)=a_\omega(\mu,l;\zeta)=0$ for $\mu>l$ and we have
\begin{eqnarray}
    V_\theta(l_1,l_2;\zeta)=\suml_{\mu}a^T_\theta(\mu,l_1;\zeta)a_\theta(\mu,l_2;\zeta),\quad V_\omega(l_1,l_2;\zeta)=\suml_{\mu}a^T_\omega(\mu,l_1;\zeta)a_\omega(\mu,l_2;\zeta).
\end{eqnarray}
Since for $\sigma>0$ the matrices $V_\theta(\zeta),V_\omega(\zeta)$ are positive-definite,  $A_\omega(\zeta),A_\theta(\zeta)$ are invertible and can be selected to be differentiable with respect to $\bTheta,\bOmega$. However, they are not unique and our results are invariant to their choice. For example, one might focus on a choice with symmetric positive definite diagonal blocks, providing a unique definition.  
 According to these definitions, we introduce the alternative process as follows:
\begin{eqnarray}\label{eq:alternative}
    \psi(\xi):=\left[\begin{array}{c}
           \suml_{\zeta\in \calZ}
           R^{\frac 1 2}(\zeta)\left(\frac{\bG_\zeta}{\sqrt{m}}\right)A_\omega(\zeta)
           +
           \suml_{\zeta\in\calZ}R(\zeta)\bTheta A^{-1}_\theta(\zeta)\left(\frac{\bH_\zeta^T\bOmega_\zeta}{m
           }+\frac{\sigma\bW_\zeta+z\bGamma_\zeta A_\omega(\zeta)}{\sqrt{m}}\right)_\tu \\
          \left[\bH_\zeta A_\theta(\zeta) +
          \bOmega_\zeta A^{-1}_\omega(\zeta)\left(\frac{\bG_\zeta^TR^{\frac 12}(\zeta)\bTheta+\sigma\bW_\zeta^T+z\bGamma_\zeta^TA_\theta(\zeta)}{\sqrt{m
          }}\right)_\tU\right]_{\zeta_i=\zeta}
    \end{array}\right],
\end{eqnarray}
 where $[\bE_\zeta]_{\zeta_i=\zeta}$ for any collection of matrices $\{\bE_\zeta\}_\zeta$, denotes a matrix with the sub-matrix $\bE_\zeta$ at the position of the rows $i$ with  $\zeta_i=\zeta$ (with the same order as in $\bOmega_\zeta$). For example, we see that $\bOmega=[\bOmega_\zeta]_{\zeta_i=\zeta}$. Moreover, for each $\zeta\in \calZ$, the matrices $\bG_\zeta,\bH_\zeta, \bW_\zeta$ and $\bGamma_\zeta$ are random with independent, standard Gaussian entries, where $\bG_\zeta, \bH_\zeta$ have the same dimensions as $\bTheta,\bOmega_\zeta$, respectively and $\bW_\zeta,\bGamma_\zeta$ are $LJ\times LJ$. These matrices can be treated as block matrices agreeing with the block structure of $A_\theta(y),A_\omega(y)$. 
 Finally, $(\ldotp)_\tu$ and $(\ldotp)_\tU$ denote the block
upper-triangular (including the diagonal blocks) and strictly 
block upper triangular 
(excluding the diagonal blocks) 
parts, respectively.
\subsection{Main Theorem}
Our main theorem does not directly address $\phi$, but a perturbation of it, given by
\begin{eqnarray}\label{eq:perturbabtion}
    &\phi'(\xi):=\phi(\xi)+\sigma\left[\begin{array}{c}
         \frac 1{\sqrt{m}}\suml_{\zeta\in\calZ}R^{\frac 1 2}(\zeta)\bU_\zeta  \\
         \bV 
    \end{array}\right]+\sqrt{\frac {1+z^2}{m
    }}\left[\begin{array}{cc}
    \suml_{\zeta\in\calZ}R(\zeta)\bTheta A_\theta^{-1}(\zeta)\left(\bGamma_\zeta A_\omega(\zeta)\right)_{\rmu}  \\
    \left[\bOmega_\zeta A_\omega^{-1}(\zeta)\left(\bGamma_\zeta^TA_{\theta}(\zeta)\right)_{\tU}\right]_{\zeta_i=\zeta}
    \end{array}\right],
\end{eqnarray}
where $\bU_\zeta$ for all $\zeta\in\calZ$ and $\bV$ consist of independent, standard Gaussian entries. Despite the complex structure of $\psi$ in \eqref{eq:alternative} and the additional terms in \eqref{eq:perturbabtion}, it is seen that the equations $\psi(\xi)+\rho_0(\xi)=0$ and $\phi'(\xi)+\rho_0(\xi)=0$ both correspond to sequential procedures with the same order $\ldots\to\theta(l)\to p(l)\to \omega(l)\to q(l)\to \theta(l+1)\to\ldots$ as the original problem $\phi(\xi)+\rho_0(\xi)=0$. Hence, these equations have unique solutions that we denoted by $\xi_\psi$ and $\xi'_\phi$, respectively. Then, we have the following result:
\begin{theorem}\label{theorem:main}
    For every $\sigma>0$ and $z\in\bbR$, the solutions $\xi_\psi$ and $\xi'_\phi$ have identical distributions. In other words, for any measurable function $h:\calB\to\bbR$, we have 
    \begin{eqnarray}
        \bbE[h(\xi_\psi)]=\bbE[h(\xi'_\phi)].
    \end{eqnarray}
\end{theorem}
When $m,n\to\infty$, the terms related to $z$ vanish and we may also let $\sigma\to 0$ to achieve $\phi'\to\phi$. This approach leads to a rigorous proof of the dynamic mean field (DMF) approximation as we discuss in section \ref{sec:DMF}. However, we observe that the terms related to $\sigma,z$ may also be eliminated in finite dimensions, which we discuss next.

\subsection{Claim: Elimination of $\sigma,z$}
In finite dimensions, applying Theorem \ref{theorem:main} to the analysis of the original dynamics (with $\phi$ and not $\phi'$) entails eliminating the additional terms in $\phi'$. We note that this situation occurs when $\sigma=0$ and $z^2=-1$. Although, this case is not well-defined within our framework, we claim that the result can be extended to this case in the following way:

\newtheorem{claim}{Claim}
\begin{claim}\label{thm:claim}
    For a test function $h:\calB\to\bbR$, denote by $H(\sigma,z)$ the value of $\bbE[h(\xi_\psi)]$ for different values of $\sigma>0$ and $z\in\bbR$. Suppose that $H(\sigma,z)$ can be analytically extended to a region of complex plain for $z$, including $z=\sqrt{-1}$. Then, we have
    \begin{eqnarray}
        \bbE[h(\xi_\phi)]=\liml_{\sigma\to 0}H\left(\sigma,z=\sqrt{-1}\right),
    \end{eqnarray}
    provided that the limit exists.
\end{claim}
Note that $H(\sigma,z)$ is an even function of $z$ on the real line and therefore lacks odd terms in its Taylor expansion at $z=0$. As a result, it remains real-valued when extended to $z^2=-1$. We do not prove Claim \ref{thm:claim}, but use this approach for obtaining expressions in our case study, i.e. obtain expressions for arbitrary $z,\sigma$ and let $\sigma\to 0$ and $z^2=-1$.

\section{Approximation Scheme in Large Problems}
In this section, we turn our attention to large values of $m,n$. In particular, we assume that they grow with a fixed relative speed $\gamma=\nicefrac nm$ and hence we only mention the dependency on $m$ in the notation. To analyze the training dynamics  \eqref{eq:p_q} and \eqref{eq:theta_omega}  in this regime, our result in Section \ref{sec:main_result} suggests to study the alternative dynamics $\psi(\xi)+\rho_0(\xi)=0$. As we observe in Theorem \ref{thm:DMF}, this approach provides the asymptotically exact dynamics when $m,n$ approach infinity. However,
under Claim \ref{thm:claim}, the relation between $\phi,\psi$ is valid for any finite value of $m,n$. Specifically, when $m,n$ are large, Claim \ref{thm:claim} leads to a method to obtain sharper approximation of the training dynamics than the limit dynamics. To explain this method, we start by explicitly writing $\psi(\xi)+\rho_0(\xi)=0$ as a sequential process. For this purpose,  we denote by $\bg_\zeta(l),\bh_\zeta(l)$ the $l^\tth$ block column of $\bG_\zeta,\bH_\zeta$, respectively. Then, we may write the alternative dynamics as follows.
\begin{eqnarray}\label{eq:dyn_calc}
    &q(l)=e(l)+\frac 1{\sqrt{m}}\tlbg(l)+\suml_{\zeta\in\calZ}\suml_{\mu\leq l}R(\zeta)\theta(\mu)
            b_\theta(\mu,l;\zeta),\nwl
    &p(l)=\left[\bone \beta(l,\zeta)+\tlbh_\zeta(l)+\suml_{\mu< l}\omega_\zeta(\mu)b_\omega(\mu,l;\zeta)\right]_{\zeta_i=\zeta},
\end{eqnarray}
where throughout the paper, $\bone$ denotes an all-one vector of suitable dimension and
\begin{eqnarray}\label{eq:param_calc}
    &\tlbg(l):=\suml_{\zeta\in\calZ}\suml_{\mu}R^{\frac 1 2}(\zeta)
    \bg_\zeta(\mu)
    a_\omega(\mu,l; \zeta),\quad \tlbh_\zeta(l):=\suml_{\mu}\bh_\zeta(\mu) a_\theta(\mu,l; \zeta)
    \nwl &e(l):=\suml_{\zeta\in\calZ}\hat x(\zeta)\frac{\bone^T\omega_\zeta(l)}{m},\quad \beta(l,\zeta):=\hat x^T(\zeta)\theta(l) 
\end{eqnarray}
and $b_\theta(\mu,l;\zeta), b_\omega(\mu,l;\zeta)$ are the blocks of the matrices $B_\theta(\zeta), B_\omega(\zeta)$, respectively given by
\begin{eqnarray}\label{eq:B_calc}
    &
    A^{-1}_\theta(\zeta)\left(\frac{\bH_\zeta^T\bOmega_\zeta}{m
           }+\frac{\sigma\bW_\zeta+z\bGamma_\zeta A_\omega(\zeta)}{\sqrt{m}}\right)_\tu,\quad 
           A^{-1}_\omega(\zeta)\left(\frac{\bG_\zeta^TR^{\frac 12}(\zeta)\bTheta+\sigma\bW_\zeta ^T+z\bGamma_\zeta^TA_\theta(\zeta)}{\sqrt{m
          }}\right)_\tU.
\end{eqnarray}

This expression, allows us to propose a fixed point iterative scheme to provide a sequence of increasingly more accurate approximations of the statistics of the dynamical system. This scheme is given in Algorithm \ref{alg:iterative}.
\begin{algorithm}[tb]
   \caption{Iterative Approximation Scheme}
   \label{alg:iterative}
\begin{algorithmic}
   \STATE {\bfseries Initiate:} Solve the DMF solution to obtain $\bTheta,\bOmega$
   \REPEAT
   \STATE Calculate $\{V_\theta(\zeta), V_\omega(\zeta)\}$ by \eqref{eq: V_calculate}.
   \STATE Calculate $\{A_\theta(\zeta), A_\omega(\zeta)\}$ by block Cholesky-type decomposition of $\{V_\theta(\zeta), V_\omega(\zeta)\}$.
   \STATE Calculate $\{\tlg_\zeta(l),\tlh_\zeta(l), e(l),\beta(l,\zeta)\}$ by \eqref{eq:param_calc}
   \STATE Calculate $\{B_\theta(\zeta), B_\omega(\zeta)\}$ by \eqref{eq:B_calc}
    \STATE Calculate the corrected dynamics $\bTheta,\bOmega$ by \eqref{eq:dyn_calc}.   
   \UNTIL{Target precision is reached}
\end{algorithmic}
\end{algorithm}

\subsection{Initialization: Dynamic Mean Field Approximation}\label{sec:DMF}
We initiate Algorithm \ref{alg:iterative} by the asymptotic dynamics of $\xi_\psi$, under the \emph{mean field} assumption, i.e. letting the parameters $A_\theta(\zeta), B_\theta(\zeta), A_\omega(\zeta), B_\omega(\zeta), e(l), \beta(l,\zeta)$ concentrate on fixed limit values, denoted by $\bar A_\theta(\zeta),\bar B_\theta(\zeta), \bar A_\omega(\zeta), \bar B_\omega(\zeta), \bar e(l), \bar \beta(l,\zeta)$, respectively. Replacing these values in \eqref{eq:dyn_calc} provides the \emph{dynamic mean field (DMF)} approximation of the training dynamics, that we denote  by $\bar\xi=(\bar\bQ,\bar\bP)$. Further, we assume that the terms related to $z$ nicely vanish and also we let $\sigma\to 0$.  Assuming that the expressions continuously extend to the limit, we simply set $\sigma=0$ in our expressions. Note that from \eqref{eq:param_calc}, $\tlbg(l),\tlbh_\zeta(l)$ become independent centered Gaussian processes with the following kernels:
\begin{eqnarray}
    &\cov[\tlbg(j,l),\tlbg(j',l')]=
    \suml_{\zeta\in\calZ}R(\zeta)\barV_{\omega, j,j'}(l,l';\zeta),\nwl
    &\cov[\tlbh_\zeta(j,l),\tlbh_{\zeta}(j',l')]=\barV_{\theta, j,j'}(l,l';\zeta)\bI,
\end{eqnarray}
where $\tlbg(j,l),\tlbh(j,l)$ are the $j^{\tth}$ column of $\tlbg(l),\tlbh(l)$, respectively and $\barV_{\omega, j,j'}(l,l';\zeta), \barV_{\theta, j,j'}(l,l';\zeta)$ are the $(j,j')$ elements of the $(l,l')$ blocks of $\barV_\omega(\zeta)=\barA_\omega(\zeta)^T\barA_\omega(\zeta)$  and $\barV_\theta(\zeta)=\barA_\theta(\zeta)^T\barA_\theta(\zeta)$, respectively. Finally, the limit values can  be calculated by the concentration points of the relations in \eqref{eq:param_calc} and \eqref{eq:B_calc}, which leads to the following chain of equations:
\begin{eqnarray}\label{eq:DMF1}
    \bar\beta(l,\zeta)=\bbE[\hat x^T(\zeta)\bar\theta(l)],\quad \bar e(l)=\suml_{\zeta\in \calZ}\hat x(\zeta)\bar\alpha(l,\zeta),\quad \bar\alpha(l,\zeta):=\bbE\left(\frac{\bone^T\bar\omega_\zeta(l)}{m}\right) 
\end{eqnarray}
and
\begin{eqnarray}\label{eq: B_mean}\label{eq:DMF2}
    &\bar B_\theta(\zeta)=\bar A^{-1}_\theta(\zeta)\bbE\left(\frac{\bH_\zeta^T\bar\bOmega_\zeta}{m
           }\right)_\tu,\quad\bar B_\omega(\zeta)=\bar A^{-1}_\omega(\zeta)\bbE\left(\frac{\bG_\zeta^TR^{\frac 12}(\zeta)\bar\bTheta}{\sqrt{m}
           }\right)_\tU,
\end{eqnarray}
where $\bar\omega(l)$ and $\bar\theta(l)$ are the block columns of $\bar\bOmega=\bOmega(\bar\bQ,\bar\bP)$ and $\bar\bTheta=\bTheta(\bar\bQ,\bar\bP)$, respectively. 
Moreover, $\bar A_\theta(\zeta),\ \bar A_\omega(\zeta)$ are respectively the upper triangular parts of the block Cholesky-type decompositions of 
\begin{eqnarray}\label{eq:DMF3}
    \bar V_\theta(\zeta)=\bbE[\bar\bTheta^TR(\zeta)\bar\bTheta],\quad \bar V_\omega(\zeta)=\frac 1m\bbE[\bar\bOmega_\zeta^T\bar\bOmega_\zeta]. 
\end{eqnarray}

Using Stein's identity (integration by part), the expressions in \eqref{eq: B_mean} can also be written in a more familiar \emph{self-consistent equation} form:
\begin{eqnarray}\label{eq:saddle}\label{eq:DMF2p}
    &\bar b_\theta(\mu,l;\zeta)=\frac 1m\left[\bbE\left(\tr\left[\frac{\partial\bar\omega_{j,\zeta}(l)}{\partial\tlh_{j'}(\mu)}\right]\right)\right]_{j',j},\quad\bar b_\omega(\mu,l; \zeta)=\frac 1{\sqrt{m}}\left[\bbE\left(\tr\left[R(\zeta)\frac{\partial\bar\theta_j(l)}{\partial\tlg_{j'}(\mu)}\right]\right)\right]_{j',j}.
\end{eqnarray}
Based on this definition and Theorem \ref{theorem:main}, we prove in the following, the convergence of $\xi_\phi$ to the DMF limit. Here, we assume that $\calB$ is equipped with the maximum of $\{\|q(l)\|_2\}_l,\{\nicefrac{\|p(l)\|_2}{\sqrt m}\}_l$  for any $\xi=(\{q(l)\}_l,\{p(l)\}_l)\in\calB$, as norm. Moreover, we assume the following:
\begin{assum}[Existence and concentration of DMFT]\label{assum:DMFT} There exists a combination of fixed parameters $A_\theta(\zeta)=\bar A_\theta(\zeta), B_\theta(\zeta)=\bar B_\theta(\zeta), A_\omega(\zeta)=\bar A_\omega(\zeta), B_\omega(\zeta)=\bar B_\omega(\zeta), e(l)=\bar e(l), \beta(l,\eta)=\bar \beta(l,\eta)$, for which the DMF dynamics in  \eqref{eq:dyn_calc} satisfies the relations in \eqref{eq:DMF1}, \eqref{eq:DMF2} (or alternatively \eqref{eq:DMF2p}) and \eqref{eq:DMF3}\footnote{If existent, such a combination is unique up to the definition of the Block-Cholesky decomposition, which is unimportant. This is because the chain of equations in \eqref{eq:DMF1}, \eqref{eq:DMF2}  and \eqref{eq:DMF3} also has a sequential nature.}. Moreover, 
    there exists a constant $c_{m}$ vanishing with $m$, such that the variables $\bar\bTheta^TR(\zeta)\bar\bTheta,\ \nicefrac 1{m}\bar\bOmega_\zeta^T\bar\bOmega_\zeta,\  \nicefrac{\bG_\zeta^TR^{\frac 12}(\zeta)\bar\bTheta}{\sqrt{m}
           },$ $ \nicefrac{\bH_\zeta^T\bOmega_\zeta}{m
           }$, $\hat x^T(\zeta)\bar\theta(l), \nicefrac{\bone^T\bar\omega_\zeta(l)}{m} $ all have a vanishing probability to deviate more than $c_m$ from their mean value.
\end{assum}
Now, we may express our result:
\begin{theorem}\label{thm:DMF}
    Suppose that $J,L,d_z$ are fixed and there is  a constant $C$ such that $\theta_l,\omega_l$ are all $C-$ Lipschitz continuous\footnote{This precisely means that $\{\theta_l\}$ and $\{\omega_l\}$ are both bounded input, bounded output (BIBO)-stable dynamical systems with the input signal $\{q(l),p(l)\}$.} and $\mu(\zeta),R(\zeta)$ are bounded by $L$ in (operator) 2-norm. Moreover, Assumption \ref{assum:DMFT} holds true. Then, for every bounded, 1-Lipschitz map $h:\calB:\to [0\ 1]$ we have 
    \begin{eqnarray}
        \left|\bbE h(\xi_\phi)-\bbE h(\bar\xi)\right|<C_{m},
    \end{eqnarray}
    where $C_{m}$ is a constant that only depends on $C$, the fixed parameters and $c_m$, and vanishes with $m$.
    \begin{proof}
        The proof is give in Section \ref{sec:proof_DMF}.
    \end{proof}
\end{theorem}

\section{Proof of Theorem \ref{theorem:main}}
In this section, we discuss the main steps involved in the proof of Theorem \ref{theorem:main}. More details are given in the appendix.
\subsection{Gordon's Comparison Lemma}
The proof of Theorem \ref{theorem:main} is based on the Gordon's comparison lemma, which allows one to compare the same statistics of two different Gaussian objects. We generalize this reult and state in the following:
\begin{theorem}[Extended Gordon Lemma]\label{lem:gordon}
    Consider two centered Gaussian random variables $\phi,\psi$ valued on a separable Banach space $(V,\|\ldotp\|)$. Take  a continuously twice Frechet-differentiable, real valued functional $f$ on $V$ with  bounded first and second order derivatives. Assume that there exists a subset $\calH$ of bounded $(0,2)-$tensors on $V$  and a constant $c$ such that $D^2f\in \calH$ and for all $H\in\calH$ we have $\bbE[H(\phi,\phi)]\leq \bbE[H(\psi,\psi)]+c$.
    Then, $\bbE[f(\psi)]\leq \bbE[f(\phi)]+\frac c 2$. 
    \begin{proof}
        The proof is given in Appendix \ref{sec:proof_gordon}.
    \end{proof}
\end{theorem}

\subsection{Gordon Lemma for zeros of Gaussian Processes}\label{sec:gordon_fix}
We attempt to apply Gordon's lemma in Theorem \ref{lem:gordon} to the analysis of the zeros of Gaussian processes. For this purpose, we take $\phi,\psi$ as two centered Gaussian processes on an open subset $U\subseteq\bbR^n$ with values on a Euclidean space $\bbR^N$ (vectorization of $\xi$) for some $N$. Then, we require to have a twice differentiable map, $\hat\xi$, assigning to their realization $\rho$ a zero $\hat\xi=\hat\xi(\rho)$, satisfying $\rho(\hat\xi)=0$. Indeed, this requirement restricts $\rho$ to a relatively small family and substantially limits the processes. However, our case of interest, regarding dynamical system, fits this framework well.


We consider the space $V=V_U$ of continuously twice differentiable functions  $\rho:U\to\bbR^N$ and define
the norm $\|\rho\|$ as the infimum of all values of $L$ for which $\rho$ and its first and second derivative are bounded by $L$ in operator norm.  
Further, we denote by $K_U$ the set of all functions  $\rho\in V$ such that $\rho$ has a unique zero $\hat\xi=\hat\xi_U(\rho)$ on $U$ and the matrix $\frac{\partial}{\partial\xi}\rho(\hat\xi)$ is invertible.
Then, we observe the following:
\begin{prop}\label{prop:deriv}
    The set $K_U$ is open in $V$ and 
     the map $\hat\xi=\hat\xi_U$ is Fr\'echet-twice differentiable on $K_U$ with
        \begin{eqnarray}
            &D\hat\xi(\rho)[\psi]=J^{-1}\psi,\nwl &D^2\hat\xi(\rho)[\psi_1,\psi_2]=J^{-1}\frac{\partial^2\rho}{\partial\xi\partial\xi^T}\left[J^{-1}\psi_1,J^{-1}\psi_2\right]+J^{-1}\frac{\partial\psi_1}{\partial\xi}J^{-1}\psi_2+ J^{-1}\frac{\partial\psi_2}{\partial\xi}J^{-1}\psi_1 
        \end{eqnarray}
        where all values are evaluated at $\xi=\hat\xi(\rho)$ and $J=-\frac{\partial\rho}{\partial\xi}$.
\end{prop}


We are to apply Gordon's lemma to functions of the form $f(\rho)=h\circ\hat\xi_U(\rho+\rho_0)$
where the test function $h:\bbR^N\to\bbR$ is twice differentiable and the fixed function $\rho_0\in V$ allows us to work with non-centered Gaussian processes, as well. We note that in this case, $D^2f[\psi,\psi]$ only depends on the products of $\psi(\xi)$ to either $\psi(\xi)$ or $\frac{\partial\psi}{\partial\xi}(\xi)$. This leads to our next result, for which we first need to introduce the following: 
\begin{defi}
    For any subset $\calM$ of $N\times N$ real matrices and any open set $U\subseteq\bbR^N$, we denote by $K_U(\calM)$ the set of all functions $\rho\in V$ with a unique zero $\hat\xi$ on $U$ such that $\frac{\partial}{\partial\xi}\rho(\hat\xi)$ is invertible and its inverse is in $\calM$. In particular, we have $K_U=K_U\left(\bbR^{N\times N}\right)$.
\end{defi}
Now, we express our result:
\begin{theorem}\label{thm:compare_equal}
    Consider two independent Gaussian processes on $\bbR^N$ and valued in $\bbR^N$,  a function $\rho_0\in V_U$ and   a set $\calM\subseteq\bbR^{N\times N}$ satisfying the following:
    \begin{enumerate}
        \item For all $t\in[0\ 1]$, the processes $\sqrt{t}\phi+\sqrt{1-t}\psi+\rho_0$ is in $K_U(\calM)$, almost surely.
        \item For all $\xi\in U$ and $M\in \calM$, it holds that
\begin{eqnarray}\label{eq:matching_gordon}
    &\bbE\left[\psi(\xi)\psi^T(\xi)\right]=\bbE\left[\phi(\xi)\phi^T(\xi)\right],\quad\bbE\left[\frac{\partial\psi}{\partial\xi}(\xi)M\psi(\xi)\right]=\bbE\left[\frac{\partial\phi}{\partial\xi}(\xi)M\phi(\xi)\right].
\end{eqnarray}
\end{enumerate}
Then, 
we have
    $\bbE[h\circ\hat\xi_U(\phi+\rho_0)]=\bbE[h\circ\hat\xi_U(\psi+\rho_0)]$,  for any twice differentiable real-valued function.
\begin{proof}
    We note that according to proposition \ref{prop:deriv}, the conditions of lemma \ref{lem:gordon} hold true with $c=0$. We neglect straightforward derivations and conclude that
    $\bbE[h\circ\hat\xi_U(\phi+\rho_0)]\leq\bbE[h\circ\hat\xi_U(\psi+\rho_0)]$.
    Then, the result is obtained by swapping $\phi,\psi$ and repeating the argument.
\end{proof}
\end{theorem}

\subsection{Proof of Theorem \ref{theorem:main}}
We arrive at  Theorem \ref{theorem:main} as a special case of Theorem \ref{thm:compare_equal} with a particular choice of $\phi=\phi'$ and $\psi$, given in Section \ref{sec:problem} and a suitable choice of $\calM$. More details are provided in Appendix \ref{sec:verify_conditions}.
\section{Example: Classification with Perceptron}
In this section, we apply our approach to the analysis of a perceptron model in a classification task. Our goal is to showcase the nature of the results obtained by our theory and hence we do not compute all variables in detail\footnote{ For the code generating the results, go to \url{https://github.com/Info2Knowledge/GC_dynamics.git}.}. For simplicity, we assume that $\zeta_i=y_i$, i.e. each class corresponds to a single Gaussian component. Hence, we replace $\zeta$ with $y$ in the notation. Moreover, we assume that $R(y)=\bI$, i.e. classes have equal covariances, but they may have different mean vectors $\hat x(y)
$. 
The perceptron model is given by a weight vector $\theta$ and an activation function $\sigma$ with derivative $\sigma'$ (which can be defined in a distribution sense, like in ReLU).
Applying the model to the samples $\{x_i\}$  leads to the feature vector $f=[f_i]_{i}\in\bbR^{m}$, given by
    $f:=\sigma(p),\quad p=\left[p_i:=x_i^T\theta\right]_{i},$
where $\sigma(\ldotp)$ is the elementwise application of the activation function $\sigma$.
To train the model $\theta$ 
we consider the empirical risk minimization framework with the cost function $L(\theta):=\nicefrac 1m\sum_i\ell(f_i,y_i
    ),$
where $\ell$ is a given loss function with the derivative $\ell'$ with respect to the first argument.  
We consider an arbitrary linear, first order (full-batch) optimization algorithm, which can be written as 
\begin{eqnarray}
    \theta(l)=\theta_0\lambda (l)-\suml_{\mu<l}q(\mu)\lambda(\mu,l),\quad q(\mu):=\nabla L(\theta(\mu)),
\end{eqnarray}
where $\theta_0$ is the initialization and $\lambda(l),\lambda(\mu,l)$ are constants describing the algorithm. For example, the momentum gradient descent (GD) algorithm with step size $t$ and momentum constant $s$ corresponds to $\lambda(l)=1$ and $\lambda(\mu,l)=t(1-s^{l-\mu})$ for $\mu<l$. The case $s=0$ corresponds to standard gradient descend.
Using matrix notation, we may rewrite this relation as 
\begin{eqnarray}\label{eq:dyn_matrix}
    \bTheta=\theta_0\blambda^T-\bQ\bLambda,
\end{eqnarray}
where $\blambda$ is the vector of $\lambda(l)$s and $\bLambda$ is the matrix of $\lambda(\mu,l)$, which is strictly upper triangular (zero diagonals).  
We may verify that at the $l^\tth$ iteration, the gradient can be written as
\begin{eqnarray}
    &q(l)=\frac 1m\suml_ix_i\omega_i(l),\quad\omega_i(l)=\omega(p_i(l),y_i),\quad\omega(p,y):=\ell'(\sigma(p),y)\sigma'(p),\
\end{eqnarray}
This representation frames the dynamics as an instance of our setup in \eqref{eq:p_q} and \eqref{eq:theta_omega}. We apply our approach to this case.
\subsection{DMF approximation}
To start, we consider the DMF approximation, where we note that the alternative dynamics in \eqref{eq:dyn_calc} is simplified to
\begin{eqnarray}\label{eq:dyn_calc_simp}
    &\bar q(l)=\bar e(l)+\tlbg(l)+\suml_{\mu\leq l}\bar\theta(\mu)
            \barb_\theta(\mu,l),\nwl
    &\bar p(l)=\left[\bone \bar\beta(l,y)+\tlbh_y(l)+\suml_{\mu< l}\bar\omega_y(\mu)\barb_\omega(\mu,l)\right]_{y_i=y}
\end{eqnarray}
where $\bar b_\theta(\mu,l)=\sum_y \bar b_\theta(\mu,l;y)$ and we note that according to \eqref{eq:saddle}, the coefficients $\bar b_\omega(\mu,l;y)$ are identical for different values of $y$, hence denoted by $\bar b_\omega(\mu,l)$. Moreover, we observe that
\begin{eqnarray}\label{eq:DMF1_Gaussians}
    &\cov[\tlbg(l),\tlbg(l')]=  \barV_{\omega}(l,l')\bI:= \frac 1m\bbE[\bar\omega^T(l)\bar\omega(l')]\bI ,\nwl
    &\cov[\tlbh_y(l),\tlbh_{y}(l')]=\barV_{\theta}(l,l')\bI=\bbE[\bar\theta^T(l)\bar\theta(l')]\bI,
\end{eqnarray}
\begin{figure}
    \begin{subfigure}{0.45\linewidth}
        \centering
        \includegraphics[width=0.95\linewidth]{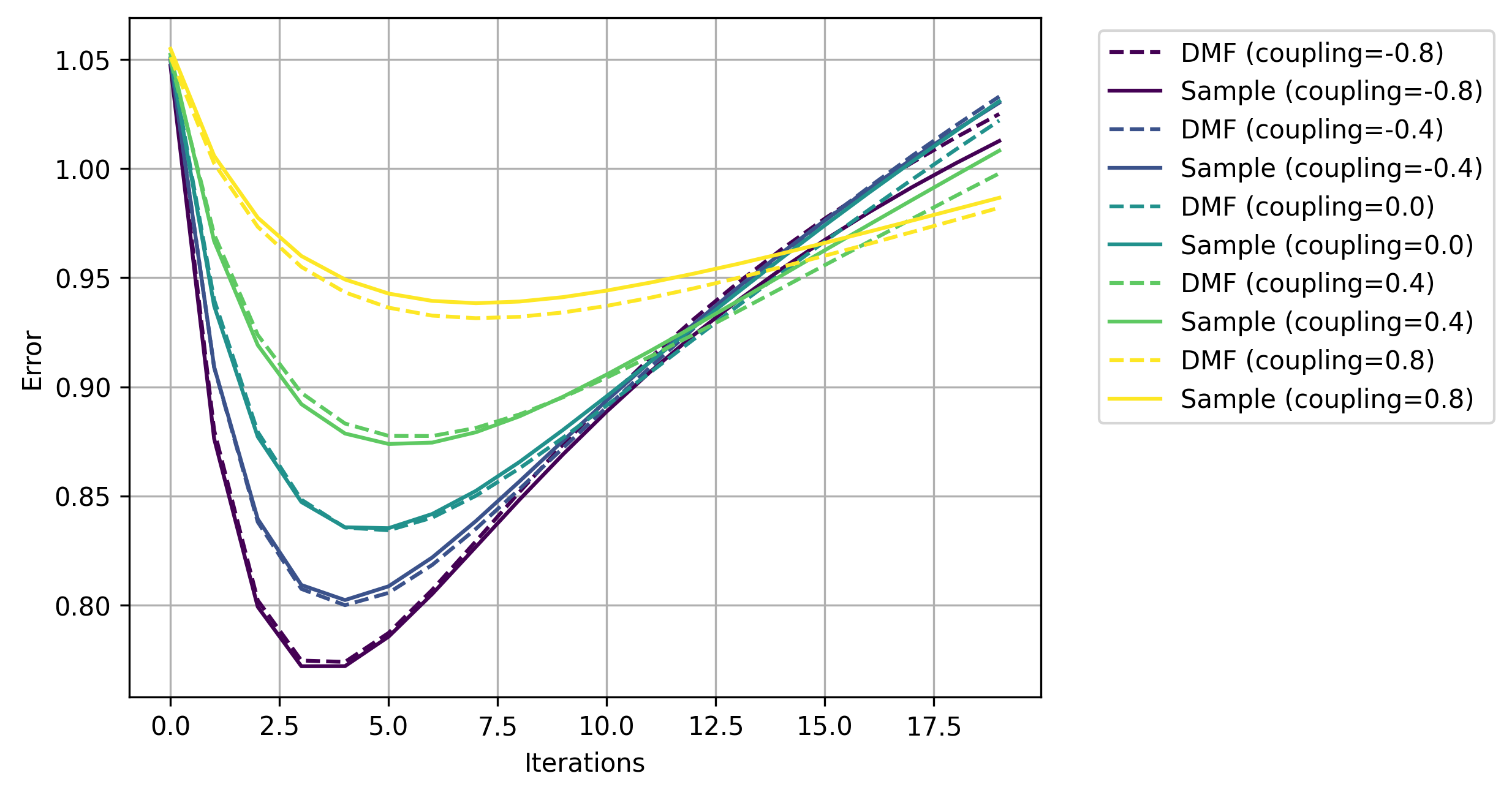}
    \caption{$\gamma=.1$ and $m=10000$}
    \label{fig:lowg_DMF}    
    \end{subfigure}
    ~
    \begin{subfigure}{0.45\linewidth}
        \centering
    \includegraphics[width=0.95\linewidth]{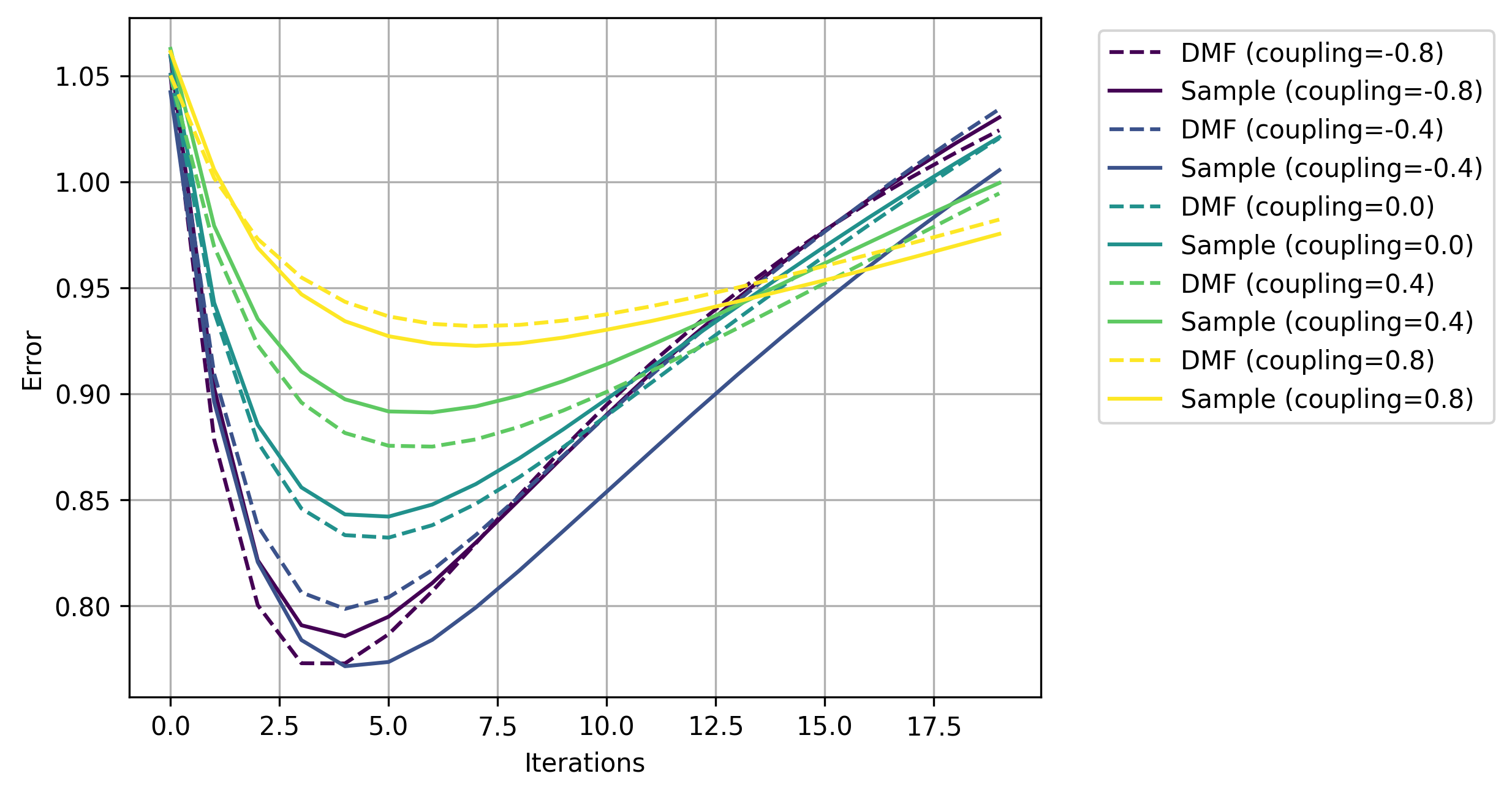}
    \caption{$\gamma=.1$ and $m=1000$}
    \label{fig:lowg_fluc}
    \end{subfigure}
    \newline
    \begin{subfigure}{0.45\linewidth}
        \centering
    \includegraphics[width=0.95\linewidth]{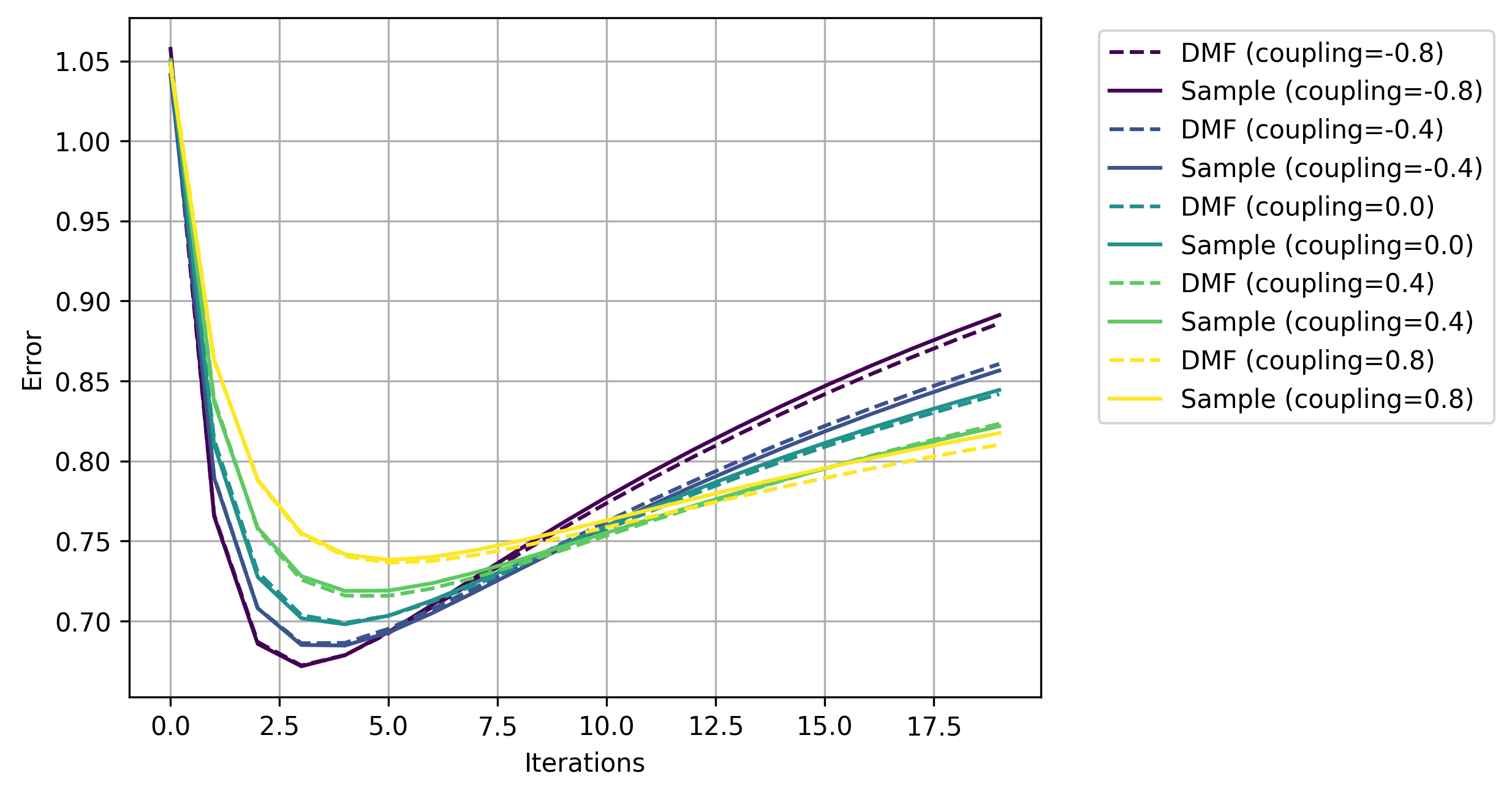}
    \caption{$\gamma=1$ and $m=10000$}
    \label{fig:over_DMF}
    \end{subfigure}
    ~
    \begin{subfigure}{0.45\linewidth}
        \centering
    \includegraphics[width=0.95\linewidth]{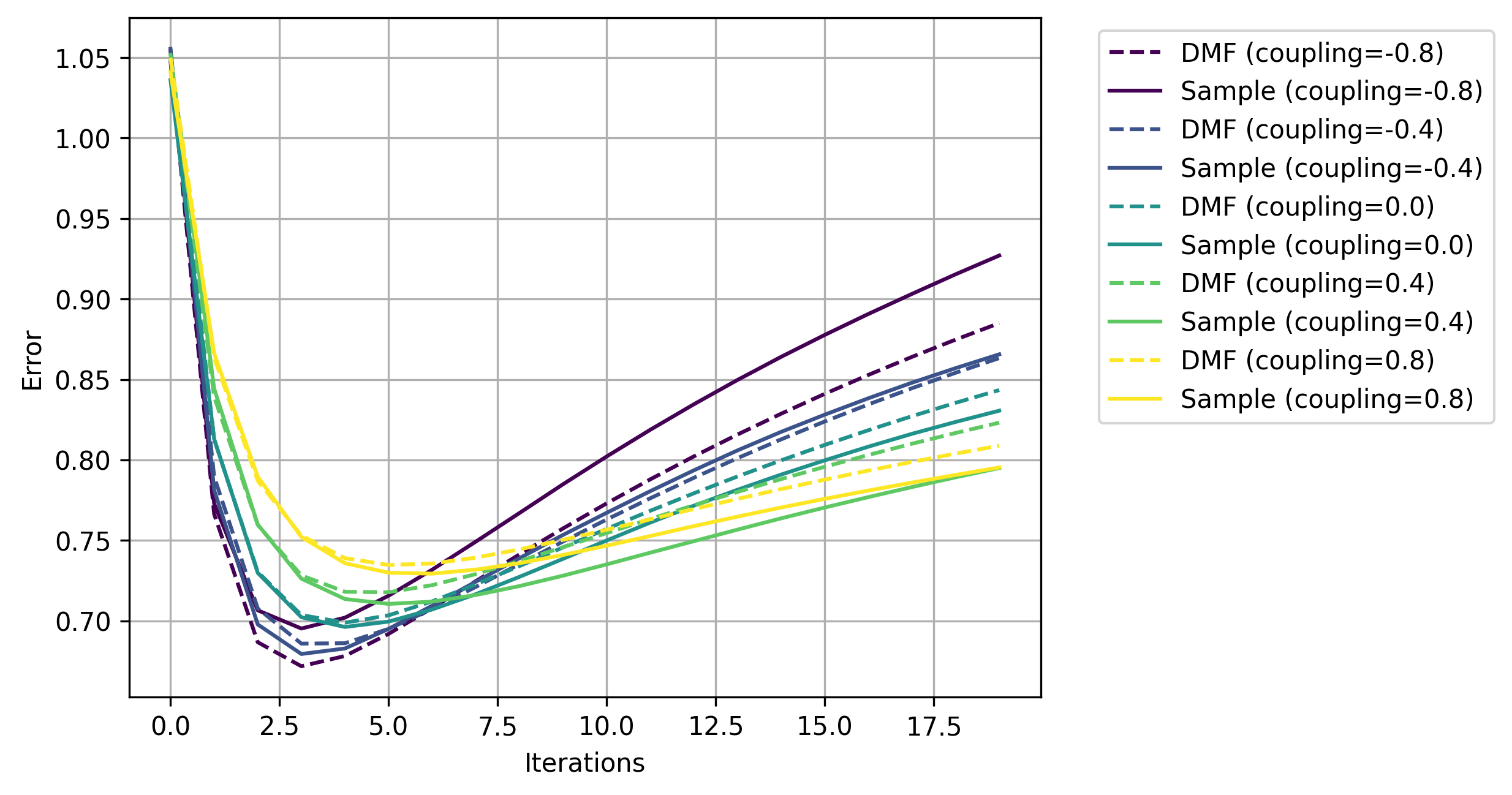}
    \caption{$\gamma=1$ and $m=1000$}
    \label{fig:over_fluc}
    \end{subfigure}
    \caption{Training error for gradient descent ($s=0, t=.2$) with two classes with $\rho(y)=.5, \|\theta_0\|=.1$, and overlaps: $v(y,y)=1$ and varying coupling$=v(0,1)$.}
\end{figure}
We write $\bar e(l)=\sum_y\hat x(y)\bar\alpha (l,y)$ where $\bar\alpha (l,y)=\bbE\left(\nicefrac{\bone^T\bar\omega_y(l)}{m}\right)$ and denote by $\bar\balpha(y)$ their vector. Based on these observations, we use \eqref{eq:dyn_matrix} to conclude that
\begin{eqnarray}
    &\bar\bTheta=\bar\Theta_\td-\frac 1{\sqrt{m}}\tlbG \tilde\bLambda,\quad \bar\Theta_\td:=\theta_0\tilde\blambda^T-\suml_{y\in\calY}\hat x(y)\tilde\balpha^T(y)=-\suml_{y\in\calY^*}\hat x(y)\tilde\balpha^T(y),
\end{eqnarray}
where $\tilde\blambda^T=\blambda^T(\bI+B_\theta\bLambda )^{-1}$ is the propagator of initialization, $\tilde\balpha^T(y)=\bar\balpha^T(\bI+\bLambda B_\theta)^{-1}\bLambda$ is the propagator of mode coefficients, $\tilde\bLambda=(\bI+\bLambda B_\theta)^{-1}\bLambda$ is the noise propagator matrix and $\tlbG$ is the matrix of $\tlbg(l)$s. Moreover, we introduce $\calY^*=\calY\cup\{*\}$ with $\hat x(*):=\theta_0$ and $\tilde\balpha(*):=-\tilde\blambda$, for simplicity. Accordingly, we have
\begin{eqnarray}\label{eq:DMF1_beta_Vt}
    &\barV_\theta=
    \suml_{(y,y')\in\calY^*\times\calY^*}\nu(y,y')\tilde\balpha(y)\tilde\balpha^T(y')+\gamma\tilde\bLambda^T \barV_\omega \tilde\bLambda,\nwl
    &\bar\beta(y)=[\bar\beta(l,y)]_l=-\suml_{y'\in\calY^*}\nu(y,y')\tilde\balpha(y'),\quad y\in\calY,
\end{eqnarray}
where $\nu(y,y')=\hat x^T(y)\hat x(y')$ are the overlaps of different modes and initialization and we remind that $\gamma=\frac{n}{m}$. Next, we note that $\bar\omega_i(l)$ for different $i$s are independent and for $y_i=y$ are identically distributed to a scalar differential equation of the following form, which is a characteristic of the DMF theory:
\begin{eqnarray}\label{eq: characteristic}
    \bar p_y(l)= \bar\beta(l,y)+\tlh(l)+\suml_{\mu< l}\omega(\bar p_y(\mu),y)\barb_\omega(\mu,l)
\end{eqnarray}
where $[\tlh(l)]_l$ is distributed by the covariance $\bar V_\theta$. Moreover, we have 
\begin{eqnarray}\label{eq:DMF1_alpha_Vw}
    &\bar\alpha(l,y)=\rho(y) \bbE[\omega(\bar p_y(l),y)],\quad \bar V_\omega(l,l')=\suml_y\bar V_\omega(l,l';y)\nwl
    &\bar V_\omega(l,l';y)=\rho(y)\bbE[\omega(\bar p_y(l),y)\omega(\bar p_y(l'),y)]
\end{eqnarray}
where $\rho(y)$ is the frequency of the samples $x_i$ with $y_i=y$. Finally, we get
\begin{eqnarray}\label{eq:DMF1_B}
&\bar B_\theta=\suml_y\bar B_\theta(y)=\bar A_\theta^{-1}\suml_y\rho(y) \bbE[h(\mu)\omega(\bar p_y(l),y)]_{\mu,l},\quad\bar B_\omega=-\gamma\tilde\bLambda
\end{eqnarray}
Solving the chain of relations in \eqref{eq:DMF1_beta_Vt},\eqref{eq:DMF1_alpha_Vw} and \eqref{eq:DMF1_B} gives the DMF approximation.
\begin{figure}
\begin{subfigure}{.45\linewidth}
    \centering
    \includegraphics[width=0.95\linewidth]{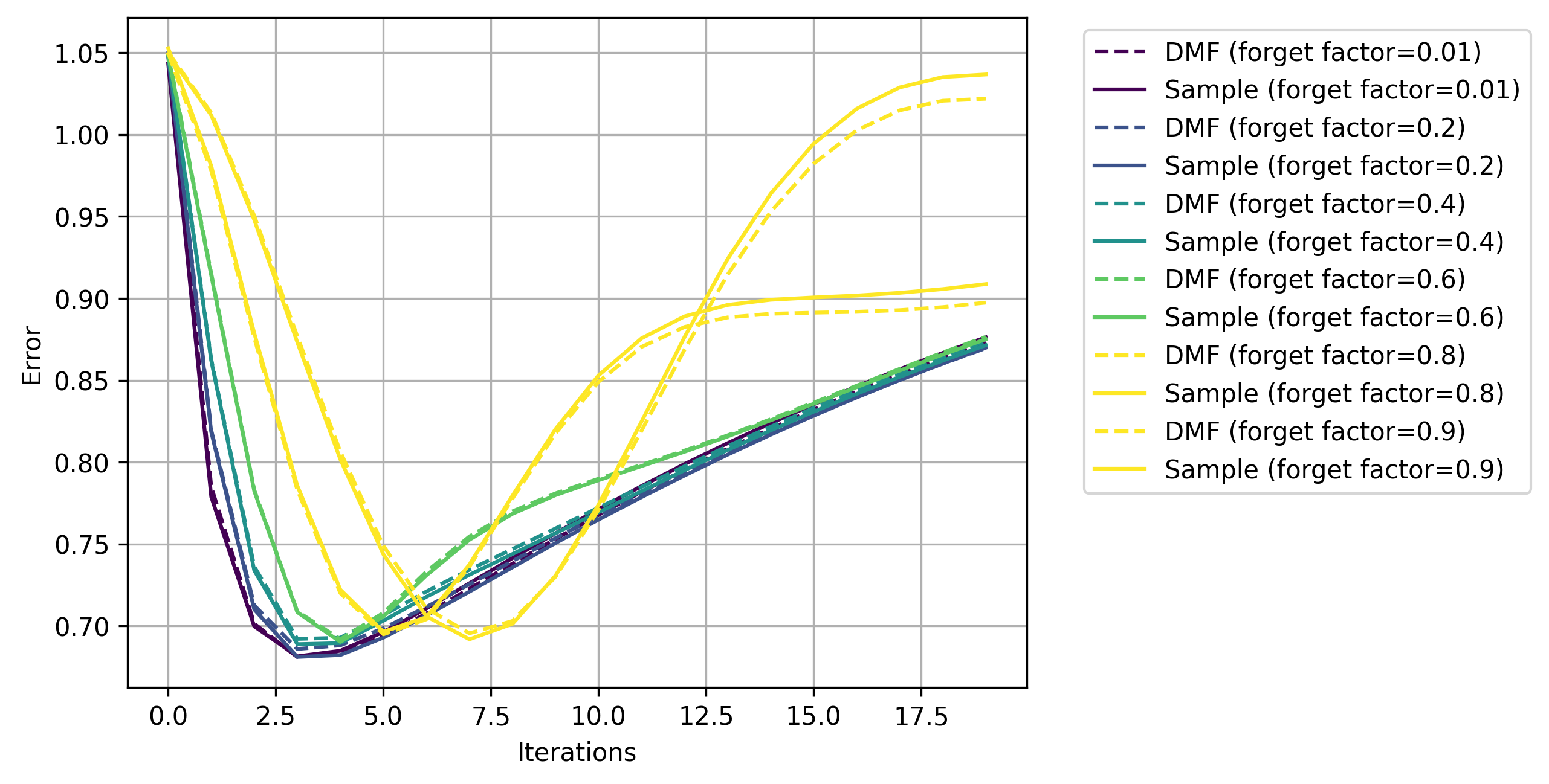}
    \caption{$m=10000$}
    \label{fig:moment_DMF}
\end{subfigure}
~
\begin{subfigure}{.45\linewidth}
    \centering
    \includegraphics[width=0.95\linewidth]{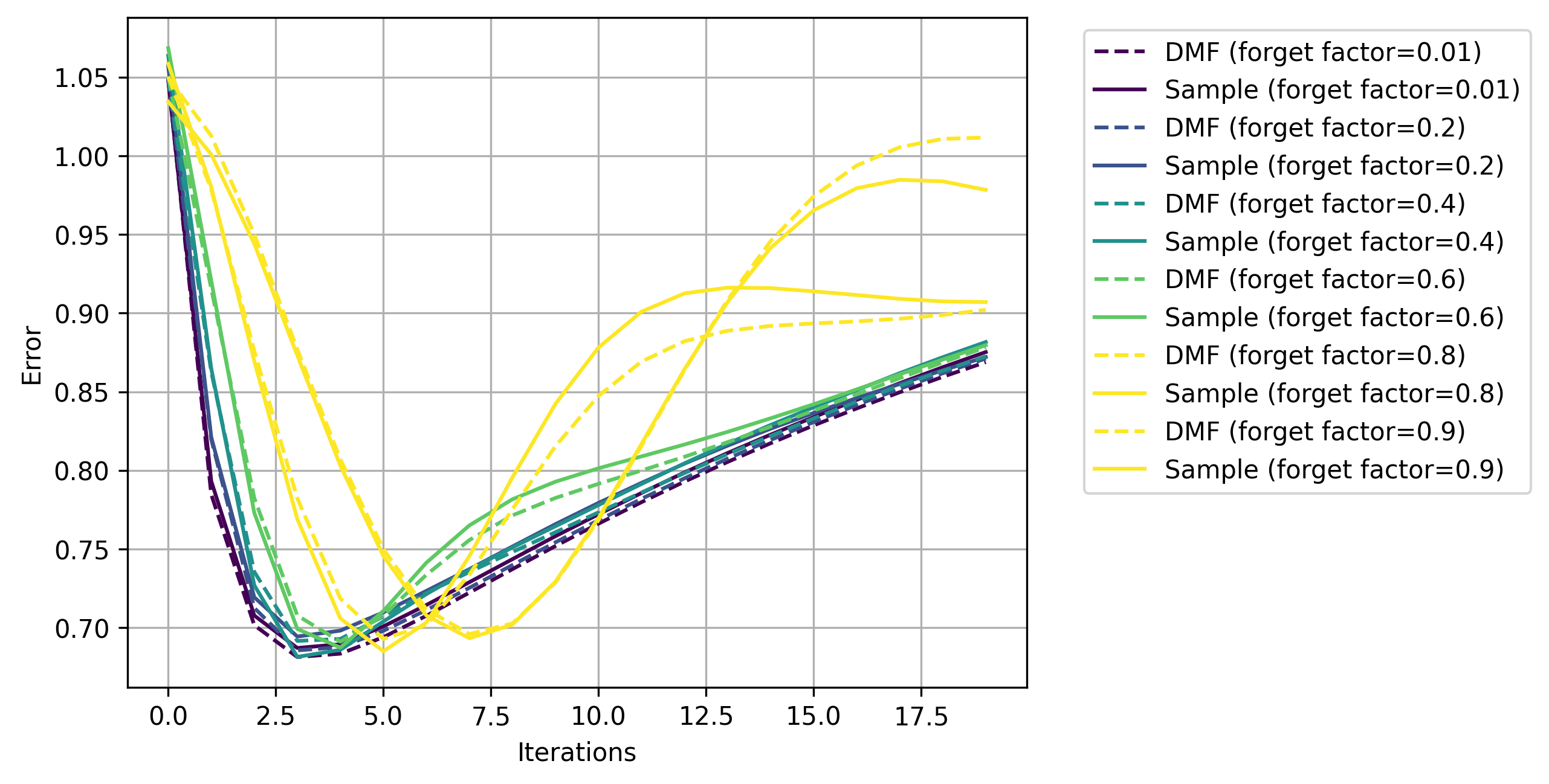}
    \caption{$m=1000$}
    \label{fig:moment_fluct}
\end{subfigure}
\caption{Training error for momentum gradient descent ($t=.2$) with varying forgetting factor $s$ and two classes with $\rho(y)=.5,\ \gamma=1,\ \|\theta_0\|=.1$, and overlaps: $v(y,y)=1$ and coupling $v(0,1)=-.5$.}
\end{figure}
\subsection{Emergence of Fluctuation Parameters}\label{sec:DMF1_correction}
Finally, we refine the DMF results by a single iteration of Algorithm \ref{alg:iterative}, where we need to recalculate the parameters by the solution of DMF, turning each parameter to a random variable. Our goal is to show that the statistics, up to $O(1/m)$, will depend on  a finite set of additional \emph{fluctuation parameters}. Therefore, we may also calculate an approximation of the parameters that match them only in first and second order statistics (mean and covariance matching), as these parts are the only expressions with $O(1/m)$ content.  
We refer to the updated values by primed parameters, 
and the matched terms
by double primes. 

First, we observe that $V'_\theta$ and hence $A'_\theta$ will only depend on $\bG_y$s, while  $V'_\omega(y)$ and hence $A'_\omega(y)$  only depend on $\bH_y$. Consequently, the updated values $\tlbG'$ and $\{\tilde\bH'_y\}$ of $\tlbG$ and $\{\tilde\bH_y\}$  still remain Gaussian as in \eqref{eq:DMF1_Gaussians}, when conditioned on $V'_\theta, V'_\omega(y)$. Hence, the refined statistics is completely identified with updated variables $V'_\theta=\bar\bTheta^T\bar\bTheta,\ V'_\omega(y)=\nicefrac 1m\bar\bOmega_y^T\bar\bOmega_y,\ \beta'(y)=\hat x^T(y)\bar\bTheta,\  \alpha'(y)=\nicefrac{\bone^T\bar\bOmega_y}m,$ and $C'_\omega(y):=\nicefrac{\bH_y^T\bar\bOmega_y}m$ and $C'_\theta(y):=\nicefrac{\bG_y^T\bar\bTheta} {\sqrt{m}}$, together with additional random variables $z\bGamma$. We may also take $\bar C_\omega(y)=\bar A_\theta\bar B_\theta(y),\  \bar C_\omega(y)=\bar A_\omega(y)\bar B_\omega$.  Then, we define the following:
\begin{defi} We introduce the fluctuation parameters $g_e^T(y)=\nicefrac 1{\sqrt{n}}\bone^T\bG_y,\ g_o^T(y,y')=\hat x^T(y')\bG_y$ and $h_e^T(y)=\nicefrac 1{\sqrt{m\rho(y)}}
\bone^T\bG_y$. Moreover, we take $\omega_e(l,y):=\omega(p_e(\mu;y),y)$ where $p_e(l,y)$ is generated by \eqref{eq: characteristic} when $\tlh(y)$ is replaced by the elements of  $\tlh^T_e(y)=h^T_e(y)\bar A_\theta$. For simplicity, we also define
    $\tlg^T_o(y)=\hat x^T(y)\tlbG\tilde\bLambda,\ \tlg^T_e=\frac 1{\sqrt{n}}\bone^T\tlbG\tilde\bLambda$, which purely depend on $g_e(y),g_o(y)$.
\end{defi}
Our final result shows that the statistics of the refined dynamics will purely depend on the above fluctuation parameters. This result can be verified by direct calculation:
\begin{prop}
    The recalculated matrices  have an identical first and second order (joint) statistics with $\{\bG_y,\bH_y\}$ to $\beta''(y)=\bar\beta(y)-\frac 1{\sqrt{m}}\tlg_o(y),\ \alpha''(y)=\bar\alpha(y)+\sqrt{\frac{\rho(y)}{m}}(w_e(y)-\frac 1{\rho(y)}\bar\alpha(y))$ and
    \begin{eqnarray}
    &{\small V''_\theta=\bar V_\theta+\frac 1{\sqrt m}\left(\left(\suml_{y\in\calY^*}\hat x(y)\tlg^T_o(y)+\tlg_o(y)\hat x^T(y)\right)+\sqrt{\gamma}\left(\tlg_e\tlg^T_e-\bar V_\theta\right)\right)},\nwl
        &V''_\omega(y)=\bar V_\omega(y)+\sqrt{\frac {\rho(y)}m} \left(\omega_e(y)\omega_e^T(y)-\frac 1{\rho(y)}\bar V_\omega(y)\right)
        \nwl &C''_\omega(y)=\bar C_\omega(y)+\sqrt{\frac{\rho(y)}{m}}\left(h_e(y)\omega_e^T(y)-\frac 1{\rho(y)}\bar C_\omega(y)\right) 
        \nwl
        &
        C''_\theta(y)=\bar C_\theta(y)-\frac 1{\sqrt{m}}\left(\suml_{y'}g_o(y,y')\tilde\alpha^T(y')\omega_e^T(y)+g_e(y)\tlg_e^T-\bar C_\omega(y)\right),
    \end{eqnarray}
\end{prop}
Figure \ref{fig:fluct} depicts the variance of the deviation from the DMF expressions, in a scenario with $m=n=2000$, where the first order terms of $O(1/m)$ dominate the statistics. Hence, the variance is normalized by multiplying to $m$. To eliminate $z$, we calculate the desired statistics $H(z)$ at $z=0,1$ and approximate $H(z=\sqrt{-1})$ by $2H(0)-H(1)$. We observed that this approach does not provide accurate result for the exact ReLU function, perhaps due to non-differentiability, which does not allow analytic extension to $z=\sqrt{-1}$. Instead, we use a differentiable approximation of ReLU (called soft ReLU) in the form of $\sigma(p)=\log(1+e^p)$, exhibiting a close relation between the results obtained by the original and alternative dynamics. Indeed due to the higher order terms involved in the approximation of $H(z=\sqrt{-1})$, the two curves can slightly differ.

\section{Literature review}



\begin{description}[style=unboxed,leftmargin=0cm]
    \item[Asymptotic Regimes] The theoretical understanding of training dynamics has emerged as a major direction in modern machine learning research. Despite limited studies in finite dimensions and through the lens of dynamical systems   
    \cite{saxe2013exact,wibisono2016variational, oymak2020toward, soltanolkotabi2018theoretical}, 
    a major breakthrough came only from the study of large models in asymptotic limits \cite{advani2020high}. Most notably, the Neural Tangent Kernel (NTK) theory analyzes networks in learning scales, where parameters remain close to initialization
    \cite{jacot2018neural, arora2019exact, lee2022neural, cao2019generalization}
    , allowing a simple description of the training dynamics by data kernels. As such,
    NTK primarily models a “lazy training” regime with limited feature evolution \cite{allen2019can, hanin2019finite, li2020learning}.
    A natural extension of NTK is in
    the mean-field regime,  where training can be described as the evolution of a probability measure over parameters
    \cite{mei2018mean, mei2019mean, chizat2018global, rotskoff2018neural}, better capturing feature learning and nonlinear parameter movement. A great advantage of this regime is that it admits statistical physics approaches. Relying on earlier parallels between  artificial NNs and physics, particularly spin glass theory
    \cite{amit1985spin, krauth1987learning, personnaz1985information}, dynamical extensions, such as state evolution and dynamical mean-field theory (DMFT) have gained a lot of interest and 
    provide recursive equations governed by a few order parameters for macroscopic observables during learning 
    \cite{gerbelot2024rigorous, mignacco2020dynamical, bordelon2022self, bordelon2023dynamics, bordelon2026disordered, fan2025dynamical}
    . Although these expressions often match the experimental results, many of them remain mathematically unproven and cannot be generalized to finite regimes. 

    \item[Gaussian Comparison] A more recent rigorous direction leverages Gaussian comparison principles, such as Slepian’s lemma \cite{slepian1962one}, Gordon’s min–max theorem \cite{gordon2006milman}, and related Gaussian process inequalities \cite{thrampoulidis2014gaussian}. 
    These tools allow complex random optimization problems to be compared with simpler Gaussian surrogates, 
    especially combined with high-dimensional probability results, such as universality\cite{thrampoulidis2018precise, panahi2017universal, bosch2023precise}. However, Gaussian comparison is limited to the global solutions of convex optimization problems. Our paper generalizes this methodology to the analysis of dynamical systems.
\end{description}
\begin{figure}
    \centering
    \includegraphics[width=0.6\linewidth]{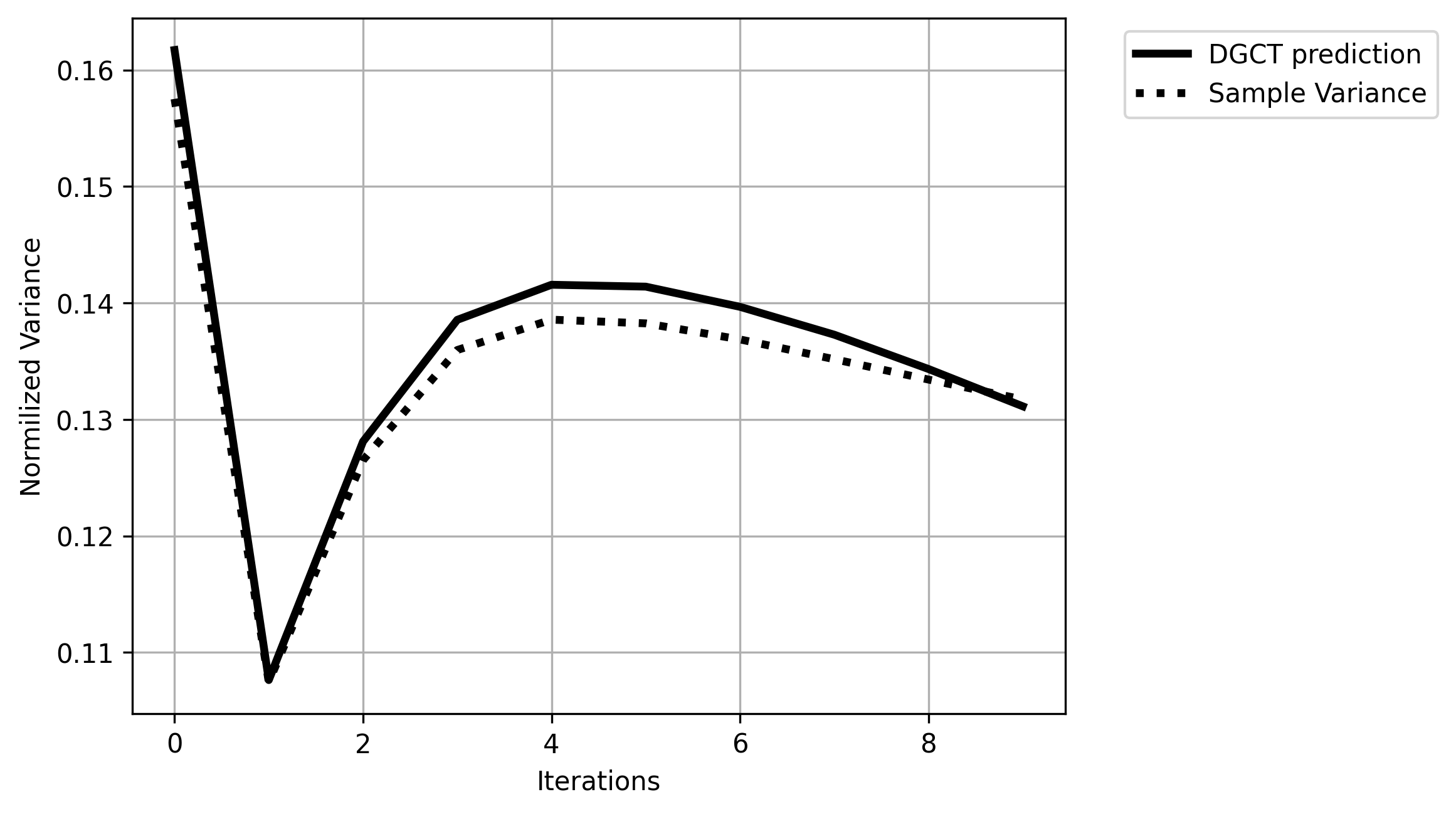}
    \caption{Normalized variance of training a perceptron with soft ReLU function, and two classes with $\rho(y)=.5,\ \gamma=1,\ \|\theta_0\|=.1$, and overlaps: $v(y,y)=1$ and coupling $v(0,1)=-.5$. The empirical values are calculated by averaging over $10^5$ realizations.}
    \label{fig:fluct}
\end{figure}
\section{Concluding remarks}
In Theorem \ref{theorem:main}, we establish a deep connection between the dynamics of training and an alternative process based on time kernels $A_\omega,A_\theta, B_\omega, B_\theta$, depending on the solution overlaps in $V_\omega,V_\theta$, as well as the bias terms $e,\beta$. This relation particularly proves the validity of the DMF limit, but more generally, shows that a similar relation holds in finite dimensions, but with random kernels. Studying finite dimensions requires eliminating the perturbation in  $\phi'$ to arrive at $\phi$, which we 
conjecture in Claim \ref{thm:claim}.  
This is the basis of our approximation scheme in Algorithm
\ref{alg:iterative}. In this scheme, we expect after $r$ iterations, the accuracy of the approximation to be improved with $O(1/m^r)$. 
The refined analysis in Section \ref{sec:DMF1_correction} agrees with this intuition as the correction terms are centered and of $O(1/\sqrt{m})$, which have only an effect of $O(1/m)$ in the statistics. In the example studied here, the correction may have a negligible effect, but when the size of kernels $A_\ldotp$ grow (e.g  by considering $J=O(m)$ or SGD scenarios where $L=O(m)$), then higher order terms become dominant and DMFT may fail. This phenomenon is an interesting direction of future study. Finally, we note that 
our theorem may exhibit a universal behavior beyond Gaussian mixtures.
\bibliographystyle{acm}
\bibliography{bibfile}

\newpage

\appendix
\section{Proof of Gordon's Lemma (Theorem \ref{lem:gordon})}\label{sec:proof_gordon}
 The proof of Theorem \ref{lem:gordon} is based on the evolution of $f$ along the interpolation path $\rho_t:=\sqrt{t}\phi+\sqrt{1-t}\psi$ between $\phi,\psi$.  To study this evolution, we introduce the following version of Stein's identity:
\begin{lemma}\label{lem:Stein}
    Take a Banch space $(V,\|\ldotp\|)$, a bounded, continuously Fr\'echet- differentiable function $v:V\to V^*$ with bounded derivative, and a centered $V-$valued Gaussian  variable $\rho$. Then, we have
     \begin{equation}
         \bbE[v(\rho)[\rho]]=\bbE[Dv(\rho)[\rho',\rho']],
     \end{equation}
     where $\rho'$ is an independent copy of $\rho$.
     \begin{proof}
         We take a sequence of continuous operators $E_n: V\to V$ with a finite-dimensional range such that 
         \begin{eqnarray}
             \rho=\liml_{n\to\infty} E_n\rho,
         \end{eqnarray}
         almost surely. In other words, $\{E_n\}$ is a finite-dimensional approximation of identity w.r.t. $\rho$. For  construction, see \cite{croix2017karhunen}. Now, we define $\rho_n=E_n\rho$ and $\rho'_n=E_n\rho'$, which are finite dimensional, centered Gaussian variables. Then according to the Stein's identity in finite dimensions (i.e. integration by parts), we have
         \begin{equation}
            \bbE[v(\rho_n)[\rho_n]]=\bbE[Dv(\rho_n)[\rho_n',\rho_n']].
         \end{equation}
         Letting $n\to\infty$ and using the dominated convergence theorem provides the result.
     \end{proof}
     
\end{lemma}
To prove Gordon's lemma, we note that under the conditions 
\begin{eqnarray}
    \bbE[f(\psi)]-\bbE[f(\phi)]=\intl_0^1\left[\frac 1{2\sqrt{t}}\bbE Df(\rho_t)[\phi]-\frac 1{2\sqrt{1-t}}\bbE Df(\rho_t)[\psi]\right]\td t,
\end{eqnarray}
where we use Fubini's theorem to exchange expectation and integration w.r.t. $t$. Next, we invoke Stein's lemma in \ref{lem:Stein} for $v(\phi)=Df(\sqrt{t}\phi+\sqrt{1-t}\psi)$ (for fixed $t,\psi$) and $v(\psi)=Df(\sqrt{t}\phi+\sqrt{1-t}\psi)$ (for fixed $t,\phi$),  in the above, and conclude that

\begin{eqnarray}
    \bbE[f(\psi)]-\bbE[f(\phi)]=\frac 12\intl_0^1\left[\bbE D^2f\left(\rho_{t}\right)[\phi',\phi']-\bbE D^2f\left(\rho_{t}\right)[\psi',\psi']\right]\td t.
\end{eqnarray}
We conclude the result by noting the assumptions and the fact that $\rho_t,\phi',\psi'$ are independent.

\section{Verifying Conditions of Theorem \ref{thm:compare_equal} for Theorem \ref{theorem:main}}\label{sec:verify_conditions}
Now, we verify the conditions of theorem \ref{thm:compare_equal} for $\phi_{\sigma,z}=\phi'$ and $\psi_{\sigma,z}=\psi$, given in \eqref{eq:perturbabtion} and \eqref{eq:alternative}, respectively and explicitly denoted by $\sigma,z$, for clarity.  
\subsection{Structure of Jacobian}
Our first step is to identify the structure of $\calM$. For this reason, we take $\rho_t:=\sqrt{t}\phi_{\sigma,z}+\sqrt{1-t}\psi_{\sigma,z}+\rho_0$ and note that the equation $\rho_t=0$ can be written as 
\begin{eqnarray}\label{eq:q_p}
    q(l)=q_l\left(\{\omega(l')\}_{l'=0}^{l}, \{\theta(l')\}_{l'=0}^{l}\right),\quad p(l)=p_l\left(\{\omega(l')\}_{l'=0}^{l-1}, \{\theta(l')\}_{l'=0}^{l}\right)
\end{eqnarray}
where
\begin{eqnarray}\label{eq:zero_as_recursion}
            &q_l=\frac 1m\bX_t\omega(l)+\sqrt{t(1+z^2)}\suml_\zeta
            \suml_{\mu\leq l}R(_\zeta )\tltheta_\zeta(\mu)\left(\frac{\bgamma_\zeta^T(\mu)\ba_\omega(l;\zeta )}{\sqrt{m}}\right)+\sqrt{1-t}\times
            \nwl
            &\suml_\zeta\left[\suml_{\mu}
            R^{\frac 1 2}(\zeta )\left(\frac{\bg_\zeta(\mu)}{\sqrt{m}}\right)a_\omega(\mu,l; \zeta )+\suml_{\mu\leq l}R(\zeta )\tltheta_\zeta(\mu)\left(\frac{\bh_\zeta^T(\mu)\omega_\zeta(l)}m+\frac{\sigma w_\zeta(\mu,l)+z\bgamma^{'T}_\zeta(\mu)\ba_\omega(l;\zeta )}{\sqrt{m}}\right) \right],\nwl
            & p_l=\bX^T_t\theta(l)+\sqrt{t}\left[\suml_{\mu< l}\tlomega_\zeta(\mu)\left(\frac{\blambda_\zeta^T(\mu)\ba_\theta(l;\zeta )}{\sqrt{m}}\right)\right]_{y_i=y}+
            \nwl
            &\sqrt{1-t}\left[\suml_{\mu}\bh_\zeta(\mu) a_\theta(\mu,l; \zeta )+ 
            \suml_{\mu< l}\tlomega_\zeta(\mu)\frac{\bg_\zeta^T(\mu)R^{\frac 12}(\zeta )\theta(l)+\sigma w_\zeta(l,\mu)+z\blambda^{'T}_\zeta(\mu)\ba_\theta(l;\zeta )}{\sqrt{m}}\right]_{y_i=y},
        \end{eqnarray}
        where $\bX_t=\sqrt{t}\tlbX+\bM$ 
        and $\tltheta_\zeta(l),\tlomega_\zeta(l)$ are block columns of $\bTheta A^{-1}_\theta(\zeta ),\ \bOmega_\zeta A^{-1}_\omega(\zeta )$, respectively. Remember that $\omega_\zeta(l)$ denotes the block sub-matrix of $\omega(l)$ corresponding to $\zeta_i=\zeta$. 

        We note that \eqref{eq:q_p} and \eqref{eq:theta_omega} define a recursive relation (with the update order of $\theta\to p\to\omega\to q$). Hence, there exists a unique solution to $\rho_t=0$. To calculate the inverse of the Jacobian $M$ of $\rho_t$, we
        calculate $M\td\eta$, where $\td\eta=\{[\td q^T_\eta(l)\ \td p^T_\eta(l)]^T\}_l$ is an arbitrary vector, by solving $\td\rho=\td\eta$. This leads to the following relation:
\begin{eqnarray}
    &\td q(l)=\td q_\eta(l)+\suml_{\mu\leq l}\frac{\partial q_l}{\partial \omega (\mu)}(\td\omega(\mu))+
    \suml_{\mu\leq l}\frac{\partial q_l}{\partial \theta (\mu)}(\td\theta(\mu))\nwl
    &
    \td p(l)=\td p_\eta(l)+\suml_{\mu< l}\frac{\partial p_l}{\partial \omega (\mu)}(\td\omega(\mu))+
    \suml_{\mu\leq l}\frac{\partial p_l}{\partial \theta (\mu)}(\td\theta(\mu))\nwl
    &\td\omega(l)=\suml_{\mu\leq l}\frac{\partial \omega_l}{\partial p(\mu)}(\td p(\mu))+
    \suml_{\mu< l}\frac{\partial \omega_l}{\partial q(\mu)}(\td q(\mu))\nwl
    &\td\theta(l)=\suml_{\mu< l}\frac{\partial \theta_l}{\partial p(\mu)}(\td p(\mu))+
    \suml_{\mu< l}\frac{\partial \theta_l}{\partial q(\mu)}(\td q(\mu))
\end{eqnarray}
which is again in a recursive form (with order $\td\theta\to\td p\to\td\omega\to\td q$). Hence $M$ exists and can be written as
\begin{eqnarray}\label{eq:set_M}
    M\td\eta=\left[\begin{array}{c}
         \suml_{\mu\leq l} M^{\rmq\rmq}_{l \mu }\td \left(q_\eta(\mu)\right)+ \suml_{\mu\leq l} M^{\rmq\rmp}_{l \mu}\left(\td p_\eta(\mu)\right)  \\
         \suml_{\mu< l} M^{\rmp\rmq}_{l \mu}\left(q_\eta(\mu)\right)+ \suml_{\mu\leq l} M^{\rmp\rmp}_{l \mu}\left(\td p_\eta(\mu)\right)
    \end{array}\right]_l
\end{eqnarray}
we take $\calM$ as the set of all such operators with an arbitrary choice of linear operators $M^{\rmq\rmq}_{l \mu }:\bbR^{n\times J}\to \bbR^{n\times J}\ , M^{\rmp\rmq}_{l \mu }:\bbR^{n\times J}\to \bbR^{m\times J}\ , M^{\rmp\rmp}_{l \mu }:\bbR^{m\times J}\to \bbR^{m\times J}$ and  $ M^{\rmq\rmp}_{l \mu }:\bbR^{m\times J}\to \bbR^{n\times J}$. Then, the above argument shows that condition 1 in Theorem \ref{thm:compare_equal} holds true.


\subsection{Moment matching conditions}
Now, we check the second condition of theorem \ref{thm:compare_equal}, which consists of two parts in \eqref{eq:matching_gordon}. For this reason, we introduce the notation
\begin{eqnarray}
    \psi=\left[\begin{array}{c}
         q_\psi(l) \\
         \left[p_{\psi,\zeta}(l)\right]_{y_i=\zeta}
    \end{array}\right]_{l},\quad \phi=\left[\begin{array}{c}
         q_\phi(l) \\
         \left[p_{\phi,\zeta}( l)\right]_{\zeta_i=\zeta}
    \end{array}\right]_{l} 
\end{eqnarray}
We may verify the first part of \eqref{eq:matching_gordon} by direct calculation, which gives the following:
\begin{eqnarray}
    &\bbE[q_\psi(l)\otimes q_\psi(l')]=\bbE[q_\phi(l)\otimes q_\phi(l')]=\nwl
    &\frac 1m\suml_\zeta V_\omega(l,l';\zeta )\otimes\left[R(\zeta )+(1+z^2)R(\zeta )\suml_{\mu\leq l,\ \mu\leq l'}\tltheta_\zeta(\mu)\tltheta^T_\zeta(\mu)R(\zeta )\right],\nwl
    &\bbE[q_\psi(l)\otimes p_{\psi,\zeta}(l')]=\bbE[q_\phi(l)\otimes p_{\phi, \zeta}(l')]=\nwl
    &\left\{\begin{array}{cc}
    \frac{1}{m}[R(\zeta )\theta(l')]\otimes\left[\omega_\zeta(l)+(1+z^2)\suml_{\mu<l'}a^T_\omega(\mu,l; \zeta )\tlomega^T_\zeta(\mu)\right] & l'\leq l \\
        \frac 1m\left[R(\zeta )\theta_\zeta(l')+(1+z^2)R(\zeta )\suml_{\mu\leq l}\tltheta(\mu)a_\theta(\mu,l'; \zeta )\right]\otimes \omega_\zeta(l) & l< l
    \end{array}\right.
    \nwl&\bbE[p_{\psi,\zeta}(l)\otimes p_{\psi, \zeta'}(l')]=\bbE[p_{\phi,\zeta}(l)\otimes p_{\phi,\zeta'}(l')]=
    \nwl &V_\theta(l,l';\zeta )\otimes\left[\bI+\frac {1+z^2}m\suml_{\mu< l,\ \mu< l'}\tlomega_\zeta(\mu)\tlomega^T_\zeta(\mu)\right]\delta_{\zeta,\zeta '},
\end{eqnarray}
where $\otimes$ denotes tensor multiplication under suitable order of indices. To check the second part of \eqref{eq:matching_gordon}, we first note that for any realization $\rho$ of $\psi$ or $\phi$, written as
\begin{eqnarray}
    \rho=\left[\begin{array}{c}
         q_\rho(l) \\
         p_{\rho}(l)
    \end{array}\right]_l 
\end{eqnarray}
and any $M\in\calM$ as defined in \eqref{eq:set_M}, we have 
\begin{eqnarray}\label{eq:matching_template}
    \frac{\partial \rho}{\partial\xi} M\rho=\left[\begin{array}{c}
         \suml_{\mu<l}\frac{\partial q_\rho(l)}{\partial q(\mu)}
         \left(\suml_{\nu\leq \mu} M^{\rmq\rmq}_{\mu \nu }q_\rho(\nu)+ \suml_{\nu\leq \mu} M^{\rmq\rmp}_{\mu \nu}p_\rho(\nu)\right)
         
           \\
          \suml_{\mu<l}\frac{\partial p_\rho(l)}{\partial q(\mu)}
         \left(\suml_{\nu\leq \mu} M^{\rmq\rmq}_{\mu \nu }q_\rho(\nu)+ \suml_{\nu\leq \mu} M^{\rmq\rmp}_{\mu \nu}p_\rho(\nu)\right)
         
    \end{array}\right]_l+\nwl
    \left[\begin{array}{c}
         \suml_{\mu<l}\frac{\partial q_\rho(l)}{\partial q(\mu)}
         \left(\suml_{\nu\leq \mu} M^{\rmq\rmq}_{\mu \nu }q_\rho(\nu)+ \suml_{\nu\leq \mu} M^{\rmq\rmp}_{\mu \nu}p_\rho(\nu)\right)
         
           \\
          \suml_{\mu<l}\frac{\partial p_\rho(l)}{\partial p(\mu)}
         \left(\suml_{\nu< \mu} M^{\rmp\rmq}_{\mu \nu}q_\rho(\nu)+ \suml_{\nu\leq \mu} M^{\rmp\rmp}_{\mu \nu}p_\rho(\nu)\right)
    \end{array}\right]_l
\end{eqnarray}
This expression allows us to verify the required condition by matching each of the terms involved. Note that each term depends on a second order moment of any of the following forms $\bbE[q_\rho(\nu)\otimes\td q_\rho(l)],\ \bbE[q_\rho(\nu)\otimes\td p_\rho(l)],\ \bbE[p_\rho(\nu)\otimes\td q_\rho(l)]$ and $\bbE[p_\rho(\nu)\otimes\td p_\rho(l)]$, where $\td q_\rho(l)$ and $\td p_\rho(l)$ are the differentials with respect to either $\td q(\mu)$ or $\td p(\mu)$. Due to tedious but straightforward calculations, we do not present all cases, but the ones demonstrating the complexities involved. Let us first consider $\bbE[q_\rho(\nu)\otimes\td q_\rho(l)]$. For $\rho=\phi$, we have 
\begin{eqnarray}\label{eq:matching_qq_phi}
    &\bbE[q_\phi(\nu)\otimes\td q_\phi(l)]=\frac 1m\suml_\zeta  R(\zeta )\otimes\left[\frac{\omega_\zeta ^T(\nu)\td\omega_\zeta (l)}m\right] + (1+z^2)\times
    \nwl
    &\left[\left(\suml_{\alpha\leq \nu\ \alpha< l}R(\zeta )\tltheta_\zeta (\alpha)\tltheta^T_\zeta (\alpha)R(\zeta )\right)\otimes(\ba^T_\omega(\nu; \zeta )\td\ba_\omega(l; \zeta ))+\right.\nwl &\left.\left(\suml_{\alpha\leq \nu\ \alpha< l}R(\zeta )\tltheta_\zeta (\alpha)\td\tltheta^T_\zeta (\alpha)R(\zeta )\right)\otimes V_\omega (\nu,l; \zeta )\right] 
\end{eqnarray}
Similarly, we achieve for $\rho=\psi$:
\begin{eqnarray}\label{eq:matching_qq_psi}
    &\bbE[q_\psi(\nu)\td q^T_\psi(l)]=\frac 1m\suml_\zeta  R(\zeta )
    \otimes(\ba^T_\omega(\nu; \zeta )\td\ba_\omega(l; \zeta ))
    + 
    \nwl
    & \suml_\zeta \left[\left(\suml_{\alpha\leq \nu\ \alpha< l}R(\zeta )\tltheta_\zeta (\alpha)\tltheta^T_\zeta (\alpha)R(\zeta )\right)\otimes
    \left(\frac{\omega_\zeta ^T(\nu)\td\omega_\zeta (l)}m\right)
    +\right.\nwl
    &\left.(1+z^2)\left(\suml_{\alpha\leq \nu\ \alpha< l}R(\zeta )\tltheta_\zeta (\alpha)\td\tltheta^T_\zeta (\alpha)R(\zeta )\right)\otimes V_\omega (\nu,l; \zeta )\right]+
    \nwl
    &z^2\left[\left(\suml_{\alpha\leq \nu\ \alpha< l}R(\zeta )\tltheta_\zeta (\alpha)\tltheta^T_\zeta (\alpha)R(\zeta )\right)\otimes(\ba^T_\omega(\nu; \zeta )\td\ba_\omega(l; \zeta )) \right]
\end{eqnarray}
Note that the expressions in \eqref{eq:matching_qq_phi} and \eqref{eq:matching_qq_psi} do not generally match. However, we notice that there are two corresponding terms in \eqref{eq:matching_template} where $\td q_\rho(l)$ accord with $\td q(\mu)$, with $\nu\leq\mu<l$ and $\td p(\mu)$ with $\nu<\mu\leq l$. In both cases, we observe that $\ba_\omega(\nu; \zeta )$ and $\omega_\zeta (\nu)$ are independent of the variable with respect of which the differentiation is taken (i.e. $q(\mu)$ or $p(\mu)$) and hence we have
\begin{eqnarray}
    \ba^T_\omega(\nu; \zeta )\td\ba_\omega(l; \zeta )=\td(\ba^T_\omega(\nu; \zeta )\ba_\omega(l; \zeta ))=\td V_\omega(\nu,l; \zeta )=\frac 1m\td(\omega_\zeta ^T(\nu)\omega_\zeta (l))=\frac 1m\omega_\zeta ^T(\nu)\td\omega_\zeta (l)
\end{eqnarray}
and hence the terms match. Next, let us study the term $\bbE[q_\rho(\nu)\otimes\td p_\rho(l)]$. For $\rho=\phi$, we have
\begin{eqnarray}
    &\bbE[q_\phi(\nu)\otimes\td p_{\phi,\zeta}(l)]=\nwl
    &\frac 1m\left[R(\zeta )\td\theta(l)\right]\otimes\omega_\zeta (\nu)+\frac {1+z^2}m\left(R(\zeta )\suml_{\alpha\leq\nu}\tltheta_\zeta (\alpha)\td a_\theta(\alpha, l;\zeta )\right)\otimes\left(\suml_{\beta<l}\tlomega_\zeta (\beta)a_\omega(\beta, \nu)\right)+\nwl
    &
    \frac {1+z^2}m\left(R(\zeta )\suml_{\alpha\leq\nu}\tltheta_\zeta (\alpha) a_\theta(\alpha, l;\zeta )\right)\otimes\left(\suml_{\beta<l}\td\tlomega_\zeta (\beta)a_\omega(\beta, \nu)\right)
\end{eqnarray}
where $p_{\phi,\zeta }(l)$ is the block of $p_\phi(l)$ with $y_i=y$. For $\rho=\psi$, we get
\begin{eqnarray}
    &\bbE[q_\psi(\nu)\otimes\td p_{\psi,\zeta }(l)]=
    \nwl
    &\frac 1m
    \left[\left(R(\zeta )\td\theta(l)\right)\otimes\left(\suml_{\beta<l}\tlomega_\zeta (\beta)a_\omega(\beta, \nu; \zeta )\right)
    +\left(R(\zeta )\theta(l)\right)\otimes\left(\suml_{\beta<l}
    \td\tlomega_\zeta (\beta) a_\omega(\beta, \nu; \zeta )\right)
    \right]
    +\nwl
    &
    \frac 1m\left(R(\zeta )\suml_{\alpha\leq\nu}\tltheta_\zeta (\alpha)\td a_\theta(\alpha, l; \zeta )\right)\otimes\omega_\zeta (\nu)+\nwl
    &\frac {z^2}m\left(R(\zeta )\suml_{\alpha\leq\nu}\tltheta_\zeta (\alpha)\td a_\theta(\alpha, l;\zeta )\right)\otimes\left(\suml_{\beta<l}\tlomega_\zeta (\beta)a_\omega(\beta, \nu)\right)+\nwl
    &
    \frac {z^2}m\left(R(\zeta )\suml_{\alpha\leq\nu}\tltheta_\zeta (\alpha) a_\theta(\alpha, l;\zeta )\right)\otimes\left(\suml_{\beta<l}\td\tlomega_\zeta (\beta)a_\omega(\beta, \nu)\right)
\end{eqnarray}
Again, the two expressions do not match, but we note that the two corresponding terms in \eqref{eq:matching_template} are associated with $dq(\mu)$ for $\nu\leq\mu< l$ and $dp(\mu)$ for $\nu<\mu< l$. Now, we make two observations. First in both cases, we are only interested in terms with $\nu<l$ and since $a_\omega(\beta, \nu)=0$ for $\beta>\nu$, the summation limit $\beta<l$ can  be eliminated. Second, we note that $a_\omega(\beta, \nu)=0$ for $\beta>\nu$ and for $\beta\leq\nu$ the vector $\tlomega_\zeta (\beta)$ is independent of the variables ($p(\mu)_{\mu>\nu}, q(\mu)_{\mu\geq \nu}$) with respect to which the differentiation is taken. As a result, all terms involving $\td\tlomega_\zeta (\beta)$ can be eliminated. Under these two considerations, the two expression match. There are two more cases to be verified, but due to the underlying symmetries, their investigation is similar to the above cases and hence are ignored.o

\section{Proof of Theorem \ref{thm:DMF}}\label{sec:proof_DMF}
The proof of Theorem \ref{thm:DMF} is based on straightforward but tedious  bounds. Hence, we do not present the details, but a sketch of the proof. We use Theorem \ref{theorem:main} with $z=0$ and $\sigma=\sigma_m\to 0$ as $m\to\infty$. Moreover, we define for a sufficiently large constant $c'$ the event $\kappa$ as follows:

\begin{enumerate}
    \item The matrices $\tlbX,\bU_\zeta,\bV,\bW_\zeta,\bG_\zeta,\bH_\zeta$ are all bounded by $c'\sqrt{m}$ in operator norm. Moreover, the matrices $\bGamma_\zeta$ are bounded by $c'$ in operator norm.

    \item The variables $\bar\bTheta^TR(\zeta)\bar\bTheta,\ \bar\bOmega_\zeta^T\bar\bOmega_\zeta,\  \nicefrac{\bG_\zeta^TR^{\frac 12}(\zeta)\bar\bTheta}{\sqrt{m}
           },$ $ \nicefrac{\bH_\zeta^T\bOmega_\zeta}{m
           }$, $\hat x^T(\zeta)\bar\theta(l), \nicefrac{\bone^T\bar\omega_\zeta(l)}{m} $ all  deviate less than $c_m$ from their mean value.
\end{enumerate}
Note that since $h$ is valued in $[0\ 1]$
\begin{eqnarray}
    \left|\bbE[h(\xi_\phi)-h(\bar\xi)]\right|\leq\left|\bbE[(h(\xi_\phi)-h(\bar\xi))\chi_\kappa]\right|+\Pr(\kappa^c)
\end{eqnarray}
Now, we make few observations:
\begin{enumerate}
    \item According to Assumption \ref{assum:DMFT} and standard random matrix results, $\Pr(\kappa^c)$ is vanishing with $m$ for a sufficiently large value of $c'$.
    \item Under $\kappa$ and for $z=0,\sigma=\sigma_m$ the value $\|\bar\xi-\bar\xi_\psi\|$ is bounded by a constant vanishing with $m$. To see this, one can invoke the sequential nature of $\bar\xi,\bar\xi_\psi$ and start by bounding $\|q_\psi(0)-\bar q(0)\|,\|p_\psi(0)-\bar p(0)\|$ and then recursively calculate bounds for $\|q_\psi(l)-\bar q(l)\|,\|p_\psi(l)-\bar p(l)\|$. This is  straightforward, but tedious to compute and the exact bound is unimportant for us. Hence, we neglect more details.  
    \item Under $\kappa$ and for $z=0,\sigma=\sigma_m$ the value $\|\xi'_\phi-\xi_\phi\|$ is bounded by a constant vanishing with $m$. To see this, we repeat the procedure in the above and invoke the sequential nature of $\xi'_\phi,\xi_\phi$ and start by bounding $\|q'_\phi(0)- q_\phi(0)\|,\|p'_\phi(0)- p_\phi(0)\|$ and then recursively calculate bounds for $\|q'_\phi(l)- q_\phi(l)\|,\|p'_\phi(l)- p_\phi(l)\|$. Again, this is a straightforward, but tedious computation and the exact bound is unimportant for us. Hence, we neglect more details.  
\end{enumerate}

Accordingly we conclude by the triangle inequality that
\begin{eqnarray}
    &\left|\bbE[h(\xi_\phi)-h(\bar\xi)]\right|\leq\left|\bbE[(h(\xi_\phi)-h(\xi'_\phi))\chi_\kappa]\right|+\left|\bbE[(h(\xi_\psi)-h(\bar\xi))\chi_\kappa]\right|+\Pr(\kappa^c)+\nwl
    &\left|\bbE[(h(\xi'_\phi)-h(\xi_\psi))\chi_\kappa]\right|
\end{eqnarray}
where according to the above observations, the three first terms are bounded by vanishing terms and for the last one we have
\begin{eqnarray}
    \left|\bbE[(h(\xi'_\phi)-h(\xi_\psi))\chi_\kappa]\right|\leq \left|\bbE[(h(\xi'_\phi)-h(\xi_\psi))]\right|+\Pr(\kappa^c)=\Pr(\kappa^c)
\end{eqnarray}
where we invoke Theorem \ref{theorem:main} to eliminate the first term on the right hand side.

\section{Extension of Gordon's Lemma}\label{sec:extension}
Beyond the example discussed in Theorem \ref{theorem:main}, the key challenge of applying Gordon's lemma to the zero points is that it is difficult to confine the Gaussian processes to suitable sets such as $K_U$, where a proper zero point exists.   
In this part, we present a result that resolves this issue by generalizing Gordon's lemma to Gaussian processes that have negligible probability outside of sets such as $K_U$. To express this result, we assume the following:
\begin{assum}\label{assum: f_first}
    We assume an open subset $K$ of $V$ and 
    take a strictly positive, bounded and continuously twice Fr\'echet differentiable function $f:K\to\bbR$ with bounded first and second order derivatives. 
    We take $f=0$ on $K^c$ whenever necessary.
\end{assum}
\begin{defi}
 Next, we define
\begin{eqnarray}
    T_t:=\left\{(\phi,\psi)\in V\times V\middle| \sqrt{t}\phi+\sqrt{1-t}\psi\in K\right\},\quad t\in[0\ 1]
\end{eqnarray}
and for $t\in(0\ 1)$ take  the following vector field on $T_t$:
\begin{eqnarray}
    v(\phi,\psi ,t):=\left(\frac 1{2\sqrt{t}}\bbE_{\phi'}[\phi'D\log f(\rho_t)[\phi']],-\frac 1{2\sqrt{1-t}}\bbE_{\psi'}\left[\psi'D\log f(\rho_t)[\psi']\right]\right)
\end{eqnarray}
where $\rho_t=\sqrt{t}\phi+\sqrt{1-t}\psi$ and the expectations are over independent copies $\phi',\psi'$ of $\phi,\psi$, respectively. Note that these expectations evaluate the covariance operators of $\phi,\psi$ at the dual vector $D\log f(\rho_t)$ and hence exist in the Bochner sense.
\end{defi}

\begin{defi}
    Now, we denote by $S$ the set of all pairs  $(\phi,\psi)$, for which there exists a differentiable path $\gamma_t=\gamma_t (\phi,\psi)\in T_t$  for $t\in [0\ 1]$ satisfying:
\begin{eqnarray}
    \gamma_t(0)=(\phi,\psi),\quad \frac{\td}{\td t}\gamma_t=v(\gamma_t,t),\quad t\in (0\ 1).
\end{eqnarray}
In other words, $S$ is the set of all pairs for which the flow trajectory of $v$ remains in $T_t$.
\end{defi}
   Note that $v$ is continuous and hence, $S$ is open. Moreover for every $t\in [0\ 1]$, the map $\gamma_t(\phi,\psi)$, known as the evolution map  of $v$, is bijective. We denote the image of this map by $S_t=\gamma_t(S)\subseteq T_t$. 
\begin{assum}\label{assum: g_first}
   Finally, we consider a continuously Fr\'echet differentiable function $g:V\times V\to [0\ 1]$ with bounded derivative, supported on $S$.  
\end{assum}   
Then, we have the following:
\begin{theorem}\label{thm: Gordon_extension_1}
    Consider two centered Gaussian measures $\phi,\psi$ on $V$ and suppose that $f,g$ satisfy Assumption \ref{assum: f_first} and \ref{assum: g_first}. Assume that there exists a subset $\calH$ of bounded $(0,2)-$tensors on $V$  and a constant $c$ such that $ D^2f\in \calH$ and for all $H\in\calH$ we have $\bbE[H(\phi,\phi)]\leq \bbE[H(\psi,\psi)]+c$. Then, we have
    \begin{eqnarray}
        \bbE\left[f(\psi)g\circ\gamma_1^{-1}(\phi,\psi)\right]\leq\bbE\left[f(\phi)g(\phi,\psi)\right]+\frac c2.
    \end{eqnarray}
    \begin{proof}
        See section \ref{sec: proof_extension_1}.
    \end{proof}
    
\end{theorem}
Our goal is to replace $g$ with the indicator function of a suitable set, from which $\phi,\psi$ have little probability to escape. In this way, we achieve approximate relations between $\bbE f(\psi)$ and $\bbE f(\phi)$. The next results describes a suitable example, but to express it, we need few more definitions. First, we make the following definition for any linear subspace $W$ of the continuous dual space $V^*$ of $V$:
\begin{eqnarray}
    \lambda(W)=\max\left\{\supl_{u\in W\mid\|u\|_{*}\leq 1}\left\|\bbE[\braket{u,\phi}\phi]\right\|,\supl_{u\in W\mid\|u\|_{*}\leq 1}\left\|\bbE[\braket{u,\psi}\psi]\right\|\right\}, 
\end{eqnarray}
where $\|\ldotp\|_{*}$ denotes the dual norm of $\|\ldotp\|$. We also define 
\begin{eqnarray}
    \tilde T:=\bigcapl_{t\in [0\ 1]}T_t
\end{eqnarray}
Then, we introduce the following notation:
\begin{defi}
    For any subset $K$ of a normed space and any $\epsilon>0$, we denote by $K_\epsilon$ the covering of $K$ by open $\epsilon-$balls. Further, we denote by $K_{-\epsilon}$ the union of all sets $K'$ satisfying $K'_\epsilon\subseteq K$. We take $K_0:=K$.
\end{defi}
Accordingly, we introduce the following definition:
\begin{defi}
    We say that a Borel set $T\subseteq V\times V$ is \emph{expanded} if there exists a family $\{g_n:V\times V\to [0\ 1]\}_n$ of continuously differentiable functions with bounded derivatives, supported on $T$ and converging pointwise to the characteristic function $\chi_T$ of $T$.
\end{defi}
Then, we have the following:
\begin{cor}\label{cor:extended_Gordon}
    Suppose that the conditions of Theorem \ref{thm: Gordon_extension_1} hold true. Moreover, $Df$ belongs to a subspace $W$ of continuous dual vectors and $\|D\log f\|_*< \beta$ for a constant $\beta$. Take an expanded measurable set $T\subseteq \tilde K_{-\lambda(W)\beta}$. Then, we have
    \begin{eqnarray}
        \bbE[f(\psi)\chi_{T}(\phi,\psi)]\leq\bbE[f(\phi)\chi_{T_{\lambda(W)\beta}}(\phi,\psi)]+\frac c2
    \end{eqnarray}
    \begin{proof}
        First, we note that the components of $v(\phi,\psi,t)$ are bounded by $\frac{\beta\lambda}{\sqrt{t}},\frac{\beta\lambda}{\sqrt{1-t}}$, respectively. This shows that $\gamma^{-1}_t(T)\subseteq T_{\lambda\beta}\subseteq \tilde K$ 
        and hence $\gamma^{-1}_1(T)\subseteq S$. This observation allows us to apply Theorem \ref{thm: Gordon_extension_1} to $f, g_n\circ\gamma_1$ and conclude that
        \begin{eqnarray}
        \bbE\left[f(\psi)g_n(\phi,\psi)\right]\leq\bbE\left[f(\phi)g_n\circ\gamma_1(\phi,\psi)\right]+\frac c2
    \end{eqnarray}
    Letting $n\to\infty$ and using dominated convergence theorem leads to:
    \begin{eqnarray}
        \bbE\left[f(\psi)\chi_T(\phi,\psi)\right]\leq\bbE\left[f(\phi)\chi_{\gamma^{-1}_1(T)}(\phi,\psi)\right]+\frac c2
    \end{eqnarray}
    We conclude the result by noting that $\gamma^{-1}_1(T)\subseteq T_{\lambda\beta}$ and $\chi_{\gamma_1^{-1}(T)}\leq\chi_{T_{\lambda\beta}}$.
    \end{proof}
\end{cor}

\subsection{Proof of Theorem \ref{thm: Gordon_extension_1}}
\label{sec: proof_extension_1}
We consider the same interpolation  $\rho_t:=\sqrt{t}\phi+\sqrt{1-t}\psi$ as in Section \ref{sec:proof_gordon}. 
We define $g_t=g\circ\gamma_t^{-1}$ on $S_t$ and $g_t=0$, otherwise. Note that $g_t$ is differentiable everywhere and is supported on $S_t$. We note that $\frac{\partial}{\partial t}(g_t\circ\gamma_t)=\frac{\partial}{\partial t}(g)=0$. Hence,
\begin{eqnarray}
	&0=\frac{\partial g_t}{\partial t}(\phi,\psi)+Dg_t(\phi,\psi)[v(\phi,\psi)]=\nwl
	&\frac{\partial g}{\partial t}(\phi,\psi)+\frac 1{2\sqrt{t}}D_\phi g(\phi,\psi)[\bbE_{\phi'}[\phi' D\log f(\rho_t)[\phi']]]-\frac 1{2\sqrt{1-t}}
	D_\psi g(\phi,\psi)[\bbE_{\psi'}[\psi' D\log f(\rho_t)[\psi']]]=
	\nwl
	&\frac{\partial g}{\partial t}(\phi,\psi)+\frac 1{2\sqrt{t}}\bbE_{\phi'}[D_\phi g(\phi,\psi)[\phi'] D\log f(\rho_t)[\phi']]-\frac 1{2\sqrt{1-t}}
	\bbE_{\psi'}[D_\psi g(\phi,\psi)[\psi'] D\log f(\rho_t)[\psi']]
\end{eqnarray}
Multiplying by $f(\rho_t)$, we may write this relation as
\begin{eqnarray}\label{eq:null_flow}
	&0=f(\rho_t)\frac{\partial g_t}{\partial t}(\phi,\psi)+	
	\nwl
	&\frac 1{2\sqrt{t}}\bbE_{\phi'}[D_\phi g(\phi,\psi)[\phi'] Df(\rho_t)[\phi']]-\frac 1{2\sqrt{1-t}}
	\bbE_{\psi'}[D_\psi g(\phi,\psi)[\psi'] Df(\rho_t)[\psi']]
\end{eqnarray}
Now, we note that
\begin{eqnarray}
    &\bbE\left[f(\psi)g_1(\phi,\psi)\right]-\bbE\left[f(\phi)g_0(\phi,\psi)\right]=\bbE\intl_0^1\frac{\partial}{\partial t}\left[f(\rho_t)g_t(\phi,\psi)\right]\td t=\nwl
    &\intl_0^1\left[\bbE f(\rho_t)\frac{\partial}{\partial t}g_t(\phi,\psi)+\frac 1{2\sqrt{t}}\bbE g_t(\phi,\psi)Df(\rho_t)[\phi]-\frac 1{2\sqrt{1-t}}\bbE g_t(\phi,\psi)Df(\rho_t)[\psi]\right]\td t
\end{eqnarray}
Invoking Stein's lemma for $v(\phi)=g_t(\phi,\psi) Df(\sqrt{t}\phi+\sqrt{1-t}\psi)$ (fixing $t$ and $\psi$) and $v(\psi)=g_t(\phi,\psi) Df(\sqrt{t}\phi+\sqrt{1-t}\psi)$ (fixing $t$ and $\phi$), we obtain:
\begin{eqnarray}
&\bbE[f(\psi)g_1(\phi,\psi)]-\bbE[f(\phi)g_0(\phi,\psi)]=\nwl
&\intl_0^1\bbE\left[\begin{array}{c}
       f(\rho_t)\frac{\partial g_t}{\partial t}(\phi,\psi)+\frac 1{2\sqrt{t}}\bbE_{\phi'}[D_\phi g(\phi,\psi)[\phi'] Df(\rho_t)[\phi']]-\\
        \frac 1{2\sqrt{1-t}}
	\bbE_{\psi'}[D_\psi g(\phi,\psi)[\psi'] Df(\rho_t)[\psi']]
\end{array}\right]\td t+
\nwl
&\frac 12\intl_0^1\bbE\left[g_t(\phi,\psi)D^2f(\rho_t)[\psi',\psi']- g_t(\phi,\psi)D^2f(\rho_t)[\phi',\phi']\right]\td t
\end{eqnarray}
 According to \eqref{eq:null_flow}, we conclude that
 \begin{eqnarray}
&\bbE[f(\psi)g_1(\phi,\psi)]-\bbE[f(\phi)g_0(\phi,\psi)]=\nwl
&\frac 12\intl_0^1\bbE\left[g_t(\phi,\psi)D^2f(\rho_t)[\psi',\psi']- g_t(\phi,\psi)D^2f(\rho_t)[\phi',\phi']\right]\td t
\end{eqnarray}
We arrive at the claim noting $\phi,\psi,\phi',\psi'$ are jointly independent.

\subsection{Application to Analysis of Zero Points}
Our next result shows how Corollary \ref{cor:extended_Gordon} can be used for the analysis of zero points of Gaussian processes.  We follow the discussion in section \ref{sec:gordon_fix} by selecting $T$ to correspond to a suitable subset of functions on $K_{U}$, for which a unique zero point $\hat\xi(\rho)$ of $\rho+\rho_0$ exists and we  take $f(\rho)=h\circ\hat\xi (\rho+\rho_0)$ as a test function $h$ applied to this zero point. For this reason, we introduce  few definitions:
\begin{defi}
    For any open set $U\subseteq \bbR^n$, we denote by $K_{U,\beta}$  the set of all functions $\rho$ on $V$ with a unique zero point $\xi$ on $U$, such that $D^2\rho(\xi)$ is bounded by $\beta$ in operator norm, and $D\rho(\xi)$ is bijective and its inverse is bounded by $\beta$ in operator norm.
\end{defi}
\begin{defi}
    We denote by $T_{U,\epsilon,\beta}$ the set of all pairs of functions $(\phi,\psi)\in V\times V$ such that for all $t\in [0\ 1]$ and $\rho\in V$ with $\|\rho\|<\epsilon$, it holds that $\sqrt{t}\phi+\sqrt{1-t}\psi+\rho+\rho_0\in K_{U,\beta}$. In other words, for all $t\in [0\ 1]$ it holds that $\sqrt{t}\phi+\sqrt{1-t}\psi+\rho_0\in (K_{U,\beta})_{-\epsilon}$.
\end{defi}
Now, we consider two centered Gaussian processes $\phi,\psi$ on $U$, valued in $\bbR^n$. Then, we define the following.

\begin{defi}
    For every point $\xi\in U$, we define the covariance kernel functions $\Phi_\xi,\Psi_\xi: U\to\bbR^{n\times n}$ as follows
    \begin{eqnarray}
        \Phi_\xi(\xi')=\bbE[\phi(\xi')\phi^T(\xi)],\quad
        \Psi_\xi(\xi')=\bbE[\psi(\xi')\psi^T(\xi)].
    \end{eqnarray}
\end{defi}

Finally, we define the following:
\begin{defi}
    We say that $h:\bbR^n\to [-1\ 1]$ is a regular test function if $Dh, D^2h$ are bounded by 1 in operator norm and there exists a point $\xi\in U$, where $h(\xi)=0$.
\end{defi}

Now, we may restate Corollary \ref{cor:extended_Gordon} for zeros of functions:

\begin{cor}\label{cor:main}
 Assume the following:
 \begin{enumerate}
     \item There exists a constant $\lambda$ such that $\Phi_\xi,\Psi_\xi$ and their two first derivatives are bounded by $\lambda$ for every $\xi\in U$. 
     \item There exist a subspace of operators $\calJ$ and constants $c_1,c_2$ such that for any point $\xi\in U$ and any $t\in [0, 1]$, the inverese of 
 \end{enumerate}
 
 the processes satisfy:
\begin{eqnarray}
    &\left\|\bbE[\phi(\xi)\phi^T(\xi)]-\bbE[\psi(\xi)\psi^T(\xi)]\right\|_*< c_1,\nwl
    &\supl_{M\mid \|M\|\leq 1} \left\|\bbE\left[\frac{\partial\phi}{\partial\xi}(\xi)[M\psi(\xi)]\right]-\bbE\left[\frac{\partial\psi}{\partial\xi}(\xi)[M\psi(\xi)]\right]\right\|< c_2,
\end{eqnarray}
where $\|\ldotp\|_*$ is the nuclear norm, i.e. the dual of the matrix operator norm. Take an expanded set $T\subseteq T_{U,\sqrt{2}\lambda\beta,\beta}$. Then, there exists a pair of random vectors $\left(\xi_\phi,\xi_\psi\right)$ which are the zero points of $\phi+\rho_0,\psi+\rho_0$, respectively, with a probability higher than $\Pr\left[(\phi,\psi)\in T\right]$ and the following holds true for every regular test function $h$:
\begin{eqnarray}
    \left|\bbE\left[e^{h\left(\xi_\phi\right)}\right]-\bbE\left[e^{h\left(\xi_\psi\right)}\right]\right|\leq {c_1(\beta^3+\beta^2)+c_2\beta^2}+\Pr\left[(\phi,\psi)\notin T\right]
\end{eqnarray}
\begin{proof}
    We define a pair $\left(\xi_\phi,\xi_\psi\right)$ of random vectors as follows. For $(\phi,\psi)\in T$, define $\xi_\phi, \xi_\psi$ as the unique zero points of $\phi+\rho_0,\psi+\rho_0$, respectively.  Otherwise, take $\xi_\phi=\xi_\psi=\xi_0$ for a constant point $\xi_0\in U^*$, satisfying $h(\xi_0)=0$. Further, we define the function $\hat\xi(\rho)$ for $\rho\in K_{U}$ as the unique zero point of $\rho$ on $U$ and define $f(\rho)=e^{h(\hat\xi(\rho+\rho_0))}$ on $K_{U,\beta}$.
    Now, note that
    \begin{eqnarray}
    &\bbE\left[e^{h\left(\xi_\phi\right)}\right]=\bbE\left[e^{h\left(\xi_\phi\right)}\chi_{T}(\phi,\psi)\right]+\bbE\left[e^{h\left(\xi_\phi\right)}\chi_{T^c}(\phi,\psi)\right]=\nwl
    &\bbE\left[e^{h(\hat\xi(\phi+\rho_0))}\chi_{T_1}(\phi,\psi)\right]+\Pr\left(T_1^c\right)
    \end{eqnarray}
Next, note that $Df(\xi)[\psi]$ is always in the form of $a^T\psi(\xi)$ for some $a\in\bbR^n$ and some evaluation point $\xi$, which allows us to take $W$ as the set of all such operators (of $\psi$). Then, it is straightforward to see that $\lambda(W)\leq\lambda$. On the other hand,  we have $D\log f(\rho)$ is bounded by $\beta$. 
Moreover, according to proposition \ref{prop:deriv} on $K_{U,\beta}$, the constant $c$ in corollary \ref{cor:extended_Gordon} is bounded by $c_1(\beta^2+\beta^3)+c_2\beta^2$. Finally, we note that $T_{\lambda\beta}\subseteq T_{U,0,\beta}=\tilde K_{U,\beta}$.
These observations allow us to invoke corollary \ref{cor:extended_Gordon} to conclude that
\begin{eqnarray}
    &\bbE\left[e^{h\left(\xi_\phi\right)}\right]\leq\bbE\left[e^{h(\hat\xi(\psi))}\chi_{T}(\phi,\psi)\right]+{c_1(\beta^3+\beta^2)+c_2\beta^2}+\Pr\left(T^c\right)
    \leq\nwl
    &\bbE\left[e^{h\left(\xi_\psi\right)}\right]+{c_1(\beta^3+\beta^2)+c_2\beta^2}+\Pr\left(T^c\right)
    \end{eqnarray}
Switching $\phi,\psi$ and repeating the argument gives the result.
\end{proof}

\end{cor}

\end{document}